\documentclass[mathpazo,online]{jml}

\usepackage{lipsum}

\usepackage[utf8]{inputenc}
\usepackage{mathrsfs,amsthm,xcolor,verbatim,bbm,amsmath,amsfonts,
amssymb,nicefrac,enumitem}
\usepackage{hyperref,bm,mathtools,xparse,etoolbox}
\usepackage[capitalise,sort]{cleveref} % clever reference

\usepackage{textcomp}
\usepackage{sidecap}

\usepackage{graphicx}%
\usepackage{multirow}%
\usepackage{amsmath,amssymb,amsfonts}%
\usepackage{amsthm}%
\usepackage{mathrsfs}%
\usepackage[title]{appendix}%
\usepackage{xcolor}%
\usepackage{textcomp}%
\usepackage{manyfoot}%
\usepackage{booktabs}%
\usepackage{algorithm}%
\usepackage{algorithmicx}%
\usepackage{algpseudocode}%
\usepackage{listings}%
\usepackage[utf8]{inputenc} % allow utf-8 input
\usepackage[T1]{fontenc}    % use 8-bit T1 fonts
\usepackage{hyperref}       % hyperlinks
\usepackage{url}            % simple URL typesetting
\usepackage{booktabs}       % professional-quality tables
\usepackage{amsfonts}       % blackboard math symbols
\usepackage{nicefrac}       % compact symbols for 1/2, etc.
\usepackage{microtype}      % microtypography
\usepackage{xcolor}         % colors

\usepackage[utf8]{inputenc}
\usepackage{amsfonts}
\usepackage{amsmath}
\usepackage{amssymb}
\usepackage{appendix}
\usepackage{float}
\usepackage{amsthm}
\usepackage{enumitem}
\usepackage{graphicx}
\usepackage{tikz}
\usepackage{pgfplots}
\usepgfplotslibrary{fillbetween}
\usepackage{xcolor}
\usepackage{thmtools}
\usepackage{bm}
\usepackage{thm-restate}
\usepackage{wrapfig}
\usepackage{slashed}

\newcommand\RR{\mathbb{R}}
\newcommand\CC{\mathbb{C}}
\newcommand\EE{\mathbb{E}}

\newcommand\cN{\mathcal{N}}

\newcommand\cW{\mathcal{W}}

\newcommand\opn[1]{\operatorname{#1}}
\newcommand\AS{\mathcal{A}}
\newcommand\binaryAS{\AS}
\newcommand\sub[1]{_{\opn{#1}}}

\newcommand\submax{\sub{max}}
\newcommand\rank{\opn{rank}}
\newcommand\sign{\opn{sign}}
\newcommand\Real{\opn{Re}}

\newcommand\del{\backslash}
\newcommand\prob{P}

\newcommand\tailsum[1]{\mathfrak{T}_{\isoF{#1}}}

\newcommand\activation{\tau}

\newcommand\xfunc{[\mathbf x]}

\newcommand\conv{*}

\newcommand\texttype[1]{\mathrm{#1}}

\newcommand\entwise{\odot} % entrywise operations
\newcommand\entp{\entwise} % entrywise product

\newcommand\Wdistr{\cW}
\newcommand\Xdistr{\rho}
\newcommand\Xnorm[1]{\|#1\|_{\Xdistr}}
\newcommand\XnormN[1]{\|#1\|_{\cN}}
\newcommand\bracket[2]{\langle#1|#2\rangle_{\Xdistr}}
\newcommand\bracketN[2]{\langle#1|#2\rangle_{\cN}}

\newcommand\LP[2]{{#1}^{\texttype{LP}(#2)}}
\newcommand\HP[2]{{#1}^{\texttype{HP}(#2)}}
\newcommand\lp{\ell}
\newcommand\lpf[1]{\lp^{(#1)}}
\newcommand\hpf[1]{\hp^{(#1)}}

\newcommand\hp{h}
\newcommand\Ddd{\opn{D}_\Xdistr}
\newcommand\DddN{\opn{D}_\cN}
\newcommand\Lt{\mathcal{L}^2}
\newcommand\Ltd{\Lt(\RR^d;\Xdistr)}
\newcommand\Ltnd{\Lt(\RR^{nd};\Xdistr_n)}

\newcommand\cxh{\tilde c}

\newcommand\isoF[1]{\widehat{#1}} % _Fourier_ transform normalized as _isometry_
\newcommand\nnF[2]{\mathfrak F_{#2}{#1}} % _Fourier non-normalized_
\newcommand\nnf[1]{\nnF{#1}{}} % _Fourier non-normalized_ without dimension subscript

\newcommand\Dtt[1]{D_{\Xdistr}^{(#1)}}
\newcommand\DttN[1]{D_{\cN}^{(#1)}}
\newcommand\Xtail{\epsilon_\Xdistr}

\newcommand\thmspace{\vspace{1em}}
\newcommand{\rev}[1]{#1}
\newenvironment{revs}{}{\ignorespacesafterend}

\begin{document}

\title{Efficient anti-symmetrization of a neural network layer by taming the sign problem}

\author[1,2]{
Nilin Abrahamsen
	\thanks{
	Corresponding author.
{\tt nilin@berkeley.edu}.
	}
}
\author[1,3]{
Lin Lin
	\thanks{
		{\tt linlin@math.berkeley.edu}.
	}
}

\affil[1]{Department of Mathematics, University of California, Berkeley}
\affil[2]{The Simons Institute for the Theory of Computing}
\affil[3]{Lawrence Berkeley National Laboratory}

%%%%% Begin Abstract %%%%%%%%%%%
\begin{abstract}
Explicit antisymmetrization of a neural network is a potential candidate for a universal function approximator for generic antisymmetric functions, which are ubiquitous in quantum physics. However, this procedure is a priori factorially costly to implement, making it impractical for large numbers of particles. The strategy also suffers from a sign problem. Namely, due to near-exact cancellation of  positive and negative contributions, the magnitude of the antisymmetrized function may be significantly smaller than before antisymmetrization. 
We show that the anti-symmetric projection of a two-layer neural network can be evaluated efficiently, opening the door to using a generic antisymmetric layer as a building block in anti-symmetric neural network Ansatzes. This approximation is effective when the sign problem is controlled, and we show that this property depends crucially the choice of activation function under standard Xavier/He initialization methods. As a consequence, using a smooth activation function requires re-scaling of the neural network weights compared to standard initializations. 
\end{abstract}
%%%%% end %%%%%%%%%%%

%%%%% Keywords %%%%%%%%%%%
% 1. Maximum 10 keywords are supported.
%
% 2. If number of keywords is less than 10, please either leave the
% extra keywords below blank, or just remove them.
%
% 3. Please input the keywords BY ORDER.
%
% 4. Please use "," after each keyword, and use "." after the last keyword.
%%%%%%
\keywordone{Fermions,}
\keywordtwo{Sign problem, }
\keywordthree{Neural quantum states, }
\keywordfour{}
\keywordfive{}

%%%% maketitle %%%%%
\maketitle

%%%%%%%%%%%%%%%%%%%%%%%%
\section{Introduction}

%\show\AND
%\newenvironment{redtext}{\color{red}}{\ignorespacesafterend}

%%%%%%%%%

Simulation of quantum chemistry from first principles depends on the accurate modeling of \emph{fermionic} system comprised of the electrons. The Pauli exclusion principle dictates that fermionic wavefunctions must be antisymmetric with respect to particle exchange. This antisymmetry poses challenges; for instance, as the number of fermions increases, the effective parameterization of such wavefunctions becomes exceedingly complex for many systems.
The antisymmetry condition also results in near-exact cancellation between positive and negative contributions when computing observables.
This leads to the so-called fermionic sign problem (FSP), which was originally discovered in quantum Monte Carlo (QMC) simulations~\cite{SorellaBaroniCarEtAl1989,LohJrGubernatisScalettarEtAl1990,Ceperley1991}. 

Over the last decade, the scientific community has witnessed a surge in the development of methods employing neural networks (NNs) as universal function approximators. This surge is due to advancements in software tools, hardware capabilities, and algorithmic improvements. These developments have had a significant impact on the modeling of fermionic systems~\cite{LuoClark2019,HanZhangE2019,HermannSchaetzleNoe2020,PfauSpencerMatthewsEtAl2020,StokesMorenoPnevmatikakisEtAl2020,LinGoldshlagerLin_vmcnet}. 

However, constructing a universal NN representation for antisymmetric functions that does not suffer from the curse of dimensionality is still an open question. In the absence of symmetry constraints, even a simple structure such as a two-layer NN can act as a universal function approximator. In theory, one could explicitly antisymmetrize such a two-layer NN to parameterize universal antisymmetric functions. Such an explicitly antisymmetrized NN structure has been recently studied in QMC calculations, which can yield effectively the exact ground state energy for small atoms and molecules~\cite{LinGoldshlagerLin_vmcnet}. However, the computational cost of this antisymmetrization procedure appears \emph{a priori} to grow factorially with the system size.

In this paper we give a procedure to \emph{efficiently} evaluate the explicit anti-symmetrization of a two-layer neural network using a quadrature procedure. This is surprising due to the factorially many terms in the definition of the anti-symmetrization. For this statement to be meaningful we require that the sign problem is controlled, meaning that the anti-symmetrization does not make the original function vanish due to cancellations. We demonstrate that with the standard Xavier/He initialization, the sign problem is controlled when the activation function in the neural network is \emph{rough}. Examples of a rough (resp. smooth) activation function in the ReLU (resp. sigmoid). Alternatively, this statement implies that to avoid the sign problem with the sigmoid activation, the weights in the first layer need to be asymptotically larger than the standard Xavier/He initializations.

%When the ratio of the magnitudes exceeds the limit set by the decimal digits of precision, the computational result becomes meaningless. This is a severe problem for practical NN applications, and the problem is exacerbated by the wide usage of single-precision floating point arithmetic operations on GPU architectures.

Among all activation functions, the exponential activation function (real or complex) plays a special role in our analysis. This is because antisymmetrizing a two-layer NN with an exponential activation function gives rise to a determinant (called a Slater determinant), which can be evaluated in polynomial time.
By exploring the Fourier representation of a (rough) activation function, we can approximately express the explicitly antisymmetrized two-layer NN as a linear combination of polynomially  (with respect to the system size and inverse precision) many Slater determinants. 
This overcomes the factorial scaling barrier, and gives rise to a polynomial-time algorithm for approximate evaluation of antisymmetrized two-layer neural networks (\cref{thm:efficient_multiplicative_error}).

\subsection{Related work}

The representation of anti-symmetric functions is extensively studied in physics, where a widely used class of Ansatzes for anti-symmetric functions takes the form of a \textit{sum of Slater determinants}.  
Slater determinants can span a dense subset of the anti-symmetric space but the representation is very inefficient.  
Indeed, even in the case of a finite single-particle state space $|\Omega|=O(n)$ we would require $\binom{|\Omega|}{n}$ Slater determinants to span the anti-symmetric space. 
\cite{zweig_towards_2022} finds certain anti-symmetric functions that cannot be efficiently approximated using a simple sum of Slater determinants, but can be effectively expressed using a more complex Ansatz called the Slater-Jastrow form.

In the machine learning literature there is a rich body of works related to permutation-invariant data, i.e., when the input data is a set \cite{zaheer_deep_2017,santoro_simple_2017,yarotsky_universal_2022,zweig_functional_nodate}. 
But the literature on \emph{anti}-symmetrized neural networks is sparse. \cite{abrahamsen2023antisymmetric} gave approximation bounds for the class of anti-symmetric functions in the Barron space, which can be viewed as the set of functions that can be expressed as infinite two-layer neural networks.

\section{Problem setting}
The wave function $\psi(x_1,\ldots,x_n)$ of a system of $n$ \emph{indistinguishable particles} in a $d$-dimensional space ($d=1,2,3$) satisfies permutation symmetry of $|\psi|$ under interchange of the $n$ inputs $x_i\in\RR^d$. \emph{Fermions} are indistinguishable particles which satisfy the Pauli exclusion principle and correspond to an \emph{antisymmetric} wave function $\psi$. For a permutation $\pi\in S_n$ with sign $(-1)^\pi$, we have $\pi(\psi)=(-1)^\pi\psi$ where we have defined $\pi(\psi):\RR^{nd}\mapsto\CC$ by $\pi(\psi)(x):=\psi(x_{\pi(1)},\ldots,x_{\pi(n)})$ for $x\in\RR^{nd}$.
For any  $f:\RR^{nd}\to\CC$ we can define its explicit antisymmetrization
\begin{equation}\label{ASdef}\AS f=\frac1{\sqrt{n!}}\sum_{\pi\in S_n}(-1)^\pi\pi(f).\end{equation}
As will be shown later, the prefactor $1/\sqrt{n!}$ is the natural scaling in the antisymmetrization process.

A function defined on $\RR^d$ is called a single-particle function. Let $\Xdistr$ be a fast-decaying probability density on $\mathbb R^d$ and let $\Xdistr_n=\Xdistr^{\otimes n}$ be a product of single-particle densities. For simplicity of analytic computation, we may take $\Xdistr$ to be the density of the standard Gaussian $\mathcal N(0,I_{d})$ (called a Gaussian envelope). We represent an $n$-particle fermionic wave function as
\begin{equation}\label{sqrtweight}\psi=\sqrt{\Xdistr_n}\entp\AS f=\AS(\sqrt{\Xdistr_n}\entp f),\end{equation}
where $\entp$ denotes multiplication of function values. The wave function should be normalized as $\|\psi\|^2=1$, where $\|\cdot\|$ is the $L^2$-norm on $\RR^{nd}$. \cref{sqrtweight} implies $\|\psi\|=\Xnorm{\AS f}$ where $\Xnorm{\cdot}$ is the norm induced by the inner product $\bracket{f}{g}=\int \bar f(x) g(x)d\Xdistr_n(x)$.

If $f=f_1\otimes\cdots\otimes f_n$ is a product of single-particle functions, then so is
$\phi=\phi_1\otimes\cdots\otimes \phi_n$ where $\phi_i=\sqrt{\Xdistr}\entp f_i$. In this case $\psi(x)=(\AS \phi)(x)$ is a determinant (called the Slater determinant) denoted by $\phi_1\wedge\cdots\wedge \phi_n$ and defined by $(\phi_1\wedge\cdots\wedge \phi_n)(x_1,\ldots,x_n)=\frac1{\sqrt{n!}}\det[(\phi_i(x_j))_{ij}]$.
The normalization in Eq. \eqref{ASdef} is such that if $\phi_i$ are orthonormal functions on $L^2(\RR^d)$ then $\psi$ is normalized by Pythagoras' theorem, $\|\psi\|=\Xnorm{\AS f}=1$.
 %then the terms $\pi(\phi)$ of \eqref{ASdef} are orthonormal in $L^2(\Xdistr_n)$ so that

By letting $f$ range over a universal class of functions on $\RR^{nd}$ we obtain a universal class of antisymmetric functions $\psi=\sqrt{\Xdistr_n}\entp\AS f$ which are not in general normalized. However, this approach has two important drawbacks \emph{a priori}:
\begin{enumerate}
    \item\label{cancitem} The procedure of normalizing $\psi$ becomes numerically unstable if cancellations in \eqref{ASdef} cause the magnitudes of $\AS f$ to be too small compared to $f$. This can be viewed as a manifestation of the fermionic sign problem in this setting.
\item The sum \eqref{ASdef} has $n!$ terms, making it in general intractable to evaluate the sum for all but small values of $n$.
\end{enumerate}

\begin{figure}\label{fig:threeactivations}
    \centering
    \includegraphics[scale=.8]{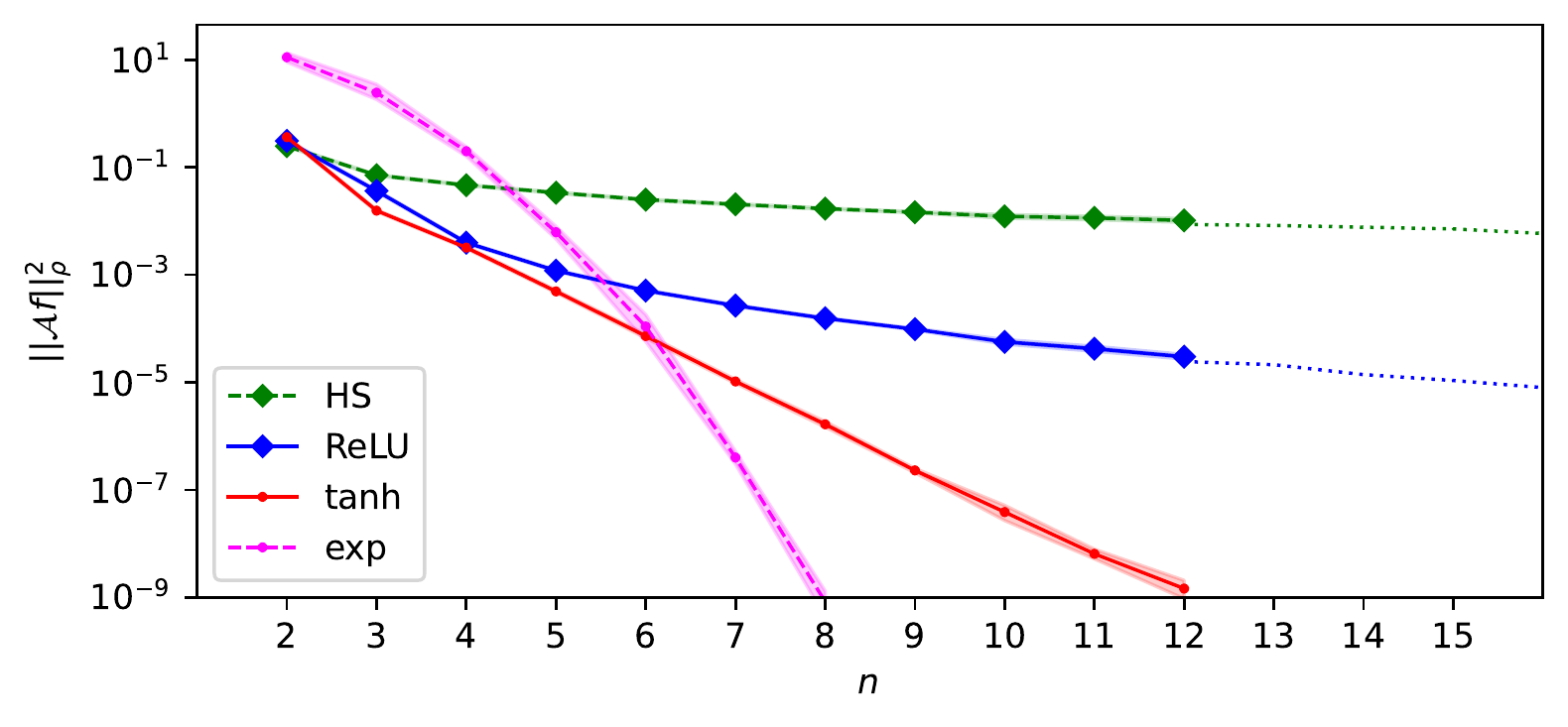}\hphantom{aaaaa}
    %\centering
    \caption{Log-plot of $\EE[\Xnorm{\AS f}^2]$ as a function of $n$ for different activation functions $\activation$: $\exp$ (magenta), $\tanh$ (red), ReLU (blue), and Heaviside step function (green). The weights are sampled from the Xavier initialization (\cref{def:He}) with $d=3$, $m=nd$. Shaded areas represent $90\%$ confidence regions. Values for $n\le 12$ are computed by direct antisymmetrization, and dotted lines on the right are computed from \eqref{Eiint}. }
\end{figure}

We consider the case when $f$ is given by a two-layer NN
\begin{equation}
f_{W,a,b}(x)=\sum_{k=1}^ma_k\activation(w^{(k)}\cdot x+b_k),
\label{NN}\end{equation}
where $\activation:\mathbb R\to\mathbb R$ is some activation function, $w^{(k)}\in\mathbb R^{nd}$, and $b_k,a_k\in\mathbb R$ for each $k=1,\ldots,m$.
%\footnote{Generating \emph{complex} amplitudes can be done in a number of ways, for example by having two output nodes (two vectors $a\sub{real},a\sub{imag}\in\RR^m$), which is equivalent with letting $a\in\CC^m$. It suffices to consider the case of a real output. Our analysis includes complex activation functions for reasons not directly related to the fact that quantum wave functions can have complex amplitudes.}. 

To illustrate the sign problem, \cref{fig:threeactivations} shows the magnitude $\Xnorm{\AS f}^2$ for four activation functions $\exp$, $\tanh$, ReLU ($\activation_{\mathrm{ReLU}}(y)=\max\{0,y\}$), and the Heaviside step function ($\activation_{\mathrm{HS}}(y)=1_{y>0}$).
As the system size $n$ increases, the norm decreases with respect to $n$ for all activation functions. However, the decay rate depends on the smoothness of the activation function.
The deterioration of the sign problem is much more severe for smooth activation functions ($\exp$, $\tanh$) than for rough activation functions (ReLU, Heaviside). 
We aim to quantify this effect and investigate more precisely how the magnitude of $\AS\activation_w$ depends on $\activation$.

\section{Main results}
\label{sec:results}

We state our results in terms of the Fourier transform $\isoF\activation$ of the activation function. A typical activation function does not have finite integral over $\RR$, so its Fourier transform is not defined as a convergent integral but rather in the sense of \emph{tempered distributions} \cite{reed_i_1981}. We will not need the precise definition of $\isoF\activation$ but only that it satisfies the \emph{Fourier inversion formula} in the sense that for $0<\epsilon<1$,
\begin{equation}
\activation(y)=\frac1{\sqrt{2\pi}}\int_{|\theta|>\epsilon}\isoF\activation(\theta)e^{i\theta y}d\theta
+p(y)+C_\epsilon+O_{\epsilon\to0}(\epsilon g(y)),
\label{Fdef}
\end{equation}
where $p$ is a polynomial of bounded degree and $g$ is bounded by a polynomial. In particular we will be able to take $p,C_\epsilon\equiv0$, $g(y)=|y|$ for $\activation=\tanh$ and $p(y)=y/2$, $C_\epsilon=\frac1{\pi\epsilon}$, $g(y)=y^2$ for $\activation=\opn{ReLU}$. The integral in \eqref{Fdef} converges for these activation functions since $\int_{|\theta|>\epsilon}|\isoF\activation(\theta)|d\theta<\infty$:

%\NA{error}

%\begin{table}[H]
\begin{figure}[H]
\caption{Fourier transforms of different activation functions. }
\label{tab:activationexamples}
\centering
{
\renewcommand{\arraystretch}{2}
\begin{tabular}{c|c|c}
$\activation(y)$&$\isoF\activation(\theta)$&Fourier tail decay $K$
\\\hline
$\opn{ReLU}(y)$&$\frac{-1}{\sqrt{2\pi}\cdot\theta^2}+\sqrt{\tfrac\pi2}i\delta'(\theta)$&$3$ (rough)
\\
$\tanh(y)$&$\frac{-i\sqrt{\pi/2}}{\sinh(\pi\theta/2)}$ &$\infty$ (smooth)
\end{tabular}
}
\end{figure}
%\end{table}

Consider the decomposition the an activation function into low- and high-frequency parts as follows:
\thmspace
\begin{definition}\label{def:hpdef}
For $\activation:\RR\to\CC$ define its \emph{high-pass} $\HP{\activation}{t}$ at threshold $t>0$ by
\begin{equation}
\label{fourieractivation}
\HP{\activation}{t}(y)=\frac1{\sqrt{2\pi}}\int_{|\theta|>t}\isoF\activation(\theta)e^{i\theta y}d\theta.
\end{equation}
Define its \emph{low-pass} as the remainder $\LP\activation{t}=\activation-\HP\activation{t}$.
\end{definition}
\thmspace
\cref{Fdef} says that $\LP\activation\epsilon=p+C_\epsilon+O(\epsilon g)$ as $\epsilon\to0$.

\thmspace
\begin{definition}\label{def:smoothrough}
For an activation function $\activation$ define its frequency tail $\tailsum{\activation}:(0,\infty)\to[0,\infty)$ by
\begin{equation}
\tailsum{\activation}(t)=\int_{|\theta|\ge t}|\isoF\activation(\theta)|^2d\theta.
\end{equation}
We define the \emph{tail decay} $K\ge0$ of $\isoF\activation$ as the largest $K$ such that $\tailsum{\activation}(t)=O(t^{-K})$ as $t\to\infty$. More precisely, $K=\limsup_{t\to\infty}\frac{-\log \tailsum\activation(t)}{\log t}$. 
\end{definition}
\thmspace

\newcommand\defineO{
\footnote{Here $\Theta(b_t)$ is to be understood as non-negative by definition.
More generally we use the standard $O(\cdot)$-notation: Write $a_t=O(b_t)$ and $b_t=\Omega(|a_t|)$ if $|a_t|\le Cb_t$ for all $t$ and some constant $C>0$. Write $\tilde b_t=\Theta(b_t)$ if $\tilde b_t=O(b_t)$ and $\tilde b_t=\Omega(b_t)$. Write $a_t=\tilde O(b_t)$ if $a_t=O(b_t|\log b_t|^{O(1)})$. Write $a_t=o(b_t)$ and $b_t=\omega(|a_t|)$ if $|a_t|=\epsilon_t b_t$ for some $\epsilon_t$ that converges to $0$.}
}

\thmspace
\begin{definition}[Smooth and rough activation functions]
\leavevmode
\begin{enumerate}
    \item 
A function $\activation:\RR\to\CC$ is \emph{smooth} if its Fourier transform $\isoF\activation$ has tail decay $\infty$, i.e., if $\tailsum\activation(t)=t^{-\omega(1)}$ decays faster than polynomially as $t\to\infty$. 
In particular, any activation function with $\isoF\activation(\theta)=\theta^{-\omega(1)}$ is smooth.
    \item
$\activation$ is \emph{rough} if there exists $k>1$ and non-zero constants $z_+,z_-\in\CC$ such that $\isoF\activation(\theta)=z_+\Theta(|\theta|^{-k})$\defineO as $\theta\to\infty$ and $\isoF\activation(\theta)=z_-\Theta(|\theta|^{-k})$ as $\theta\to-\infty$. In this case $\activation$ has Fourier tail decay $K=2k-1$.
\end{enumerate}
\end{definition}

\thmspace
Our results below for rough activation functions hold for a more general definition of roughness which allows the Fourier transform to have varying phase. We give this definition of \emph{generalized rough} activation functions in \cref{sec:generalrough}.

We consider the two-layer network \eqref{NN} with randomly initialized weights using two standard initialization strategies. It is typical to initialize the biases to zero.
\thmspace
\begin{definition}\label{def:He}
We say that $f_{W,a}(x)=\sum_{k=1}^m a_k\activation(w^{(k)}\cdot x)$ is chosen with the \emph{Xavier initialization} or \emph{He initialization} if the $mnd+m$ weights  $W=(w^{(k)}_{ij}),a=(a_i)$ are chosen independently from $ a_k\sim\cN(0,\cxh/m)$ for each $k=1,\ldots,m$, and $w^{(k)}_{ij}\sim\cN(0,\cxh/(nd))$ for $k=1,\ldots,m$, $i=1,\ldots,n$, and $j=1,\ldots,d$. Here $\cxh=1$ corresponds to the Xavier initialization and $\cxh=2$ to the He initialization.
\end{definition}

\thmspace
\begin{restatable}[Upper bound, sign problem deteriorates super-polynomially for smooth activation functions]{theorem}{Smooththm}
\label{thm:smooththm}
Let $f_{W,a}:\RR^{nd}\to\CC$ be given by a two-layer neural network \eqref{NN} with activation function $\activation$ and let $\Xdistr=\cN(0,I_{nd})$ be the Gaussian envelope function. If $\activation$ is smooth and $f_{W,a}$ is sampled from the Xavier or He initializations, then %the magnitude of the antisymmetrized NN decays super-algebraically in $n$, that is, 
with probability $1-o(1)$ over $W$, $\EE_{a|W}[\XnormN{\AS f_{W,a}}^2]=n^{-\omega(1)}$.
\end{restatable}
\thmspace

Given an activation function $\activation:\RR\to\CC$ and weight vector $w\in\RR^{nd}$ define $\activation_w:\RR^{nd}\to\CC$ by $\activation_w(x)=\activation(w\cdot x)$. Let $\lp_w=\LP\activation{t}_w$ be the low-passed part of the activation function. 
%The main step of the proof of the upper bound of $\XnormN{\AS f}^2$ is to show that the norm of the antisymmetrized function $\XnormN{\AS \lp_w}^2$ decreases exponentially with respect to $n$ when $t=O(\sqrt{n/\log n})$ (\cref{lem:lpbound}).
%The high-passed part can be directly controlled by the tail decay of the activation function.
The upper bound of $\XnormN{\AS f}^2$ relies on the fact that the norm of the antisymmetrized function $\XnormN{\AS \lp_w}^2$ decreases exponentially with respect to $n$ when $t=O(\sqrt{n/\log n})$ (\cref{lem:lpbound}).
The high-passed part can be directly controlled by the tail decay of the activation function.

\thmspace
\begin{theorem}[Lower bound, sign problem deteriorates at most polynomially for rough activation functions]
\label{thm:roughthm}
Let $f_{W,a}:\RR^{nd}\to\CC$ be given by a two-layer neural network \eqref{NN} with activation function $\activation$ and weights sampled from the Xavier or He initialization, and let $\Xdistr=\cN(0,I_{nd})$ be the Gaussian envelope. If $\activation$ is rough or generalized rough with tail decay $K$, then with probability $1-o(1)$ over $W$, $\EE_{a|W}[\XnormN{\AS f_{W,a}}^2]=\tilde \Omega(n^{-(1+2/d)K})$ where $\tilde\Omega$ denotes a lower bound up to $\log$-factors. In particular the magnitude of the antisymmetrized NN decays no faster than polynomially in $n$.
\end{theorem}
\thmspace

To prove \cref{thm:roughthm} we show that for the high-passed part $\hp=\HP\activation{T}$ at sufficiently large threshold, the norm of the antisymmetized function $\XnormN{\AS\hp_w}^2$ can be approximated by the norm $\XnormN{\hp_w}^2$ before antisymmetrization (\cref{lem:highpass_A_or_not}).
The lower bound with polynomial scaling in $n$ can also be viewed as evidence that the prefactor $1/\sqrt{n!}$ in \cref{ASdef} is the appropriate scaling.

We then show that when $f$ is given by a two-layer NN with a rough activation function, $\AS f$ can be computed efficiently to any inverse-polynomial precision relative to $\sqrt{\EE[\XnormN{\AS f}^2]}$: 

\thmspace
\begin{theorem}[Deterministic polynomial time algorithm for approximate evaluation of explicitly antisymmetrized two-layer neural network]
\label{thm:efficient_multiplicative_error}
Let $f_{W,a}:\RR^{nd}\to\CC$ be given by a two-layer neural network as in \cref{thm:roughthm} with a rough or generalized rough activation function, and let $\epsilon=n^{-O(1)}$. There exists a deterministic polynomial-time algorithm $(W,a,b,x)\mapsto S_{W,a,b}(x)$ whose output is exactly antisymmetric in $x$ and such that with probability $1-o(1)$ over $W$,
\[\EE_{a|W}[\XnormN{S_{W,a,0}-\AS f_{W,a}}^2]\le \epsilon\:\EE_{a|W}[\XnormN{\AS f_{W,a}}^2].\]
\end{theorem}
\thmspace

We consider the examples of activation functions in \cref{tab:activationexamples}.  
%\begin{enumerate}
%\item
$\tanh$ is smooth, so $\XnormN{\AS f}^2$ decays super-polynomially with $n$ with the $\tanh$ or sigmoid activation function. 
%\item 
The ReLU activation is rough with tail decay $3$. By \cref{thm:roughthm}, $\XnormN{\AS f}^2$ is of order $\tilde\Omega(n^{-(3+6/d)})$ with the ReLU activation function. By \cref{thm:efficient_multiplicative_error} there exists an efficient algorithm to compute the output of $f$ with inverse polynomial relative error when the activation function is chosen to be ReLU.

\begin{revs}
The approximation algorithm in theorem \ref{thm:efficient_multiplicative_error} involves the approximate evaluation of an integral over frequencies $\theta$.
\thmspace
\begin{remark}\label{theremark}
In the setting of the standard Xavier/He initializations, Theorems \ref{thm:smooththm}-\ref{thm:efficient_multiplicative_error} show that a rough activation function is required to avoid the sign problem. On the other hand it is still desirable to use smooth activation functions to obtain a smooth wavefunction. In the context of a smooth activation function, our results show that an initialization should be used in which the weights in the first layer are larger than those on the typical Xavier/He initializations. This re-scaling of the first layer need only be by an algebraic\footnote{This is because the re-scaling factor can be chosen as the frequency threshold %in lemma \ref{lowerlemma} 
used to prove theorem \ref{thm:roughthm}.} factor $r=n^{O(1)}$.
 In this setting the approximation algorithm of theorem \ref{thm:efficient_multiplicative_error} is unchanged except that the infra-red truncation $t$ of the integral is replaced by $t/r$.
\end{remark}
\thmspace
Using the smooth activation function $\tanh$ as an example, the modification in remark \ref{theremark} is equivalent to replacing $\tanh$ with $y\mapsto\tanh(ry)$ where $r$ grows with $n$.
\end{revs}

\section{Reduction and generic weights}
\label{sec:proofs}

We now present an outline of the proofs of the main theorems. Additional details follow in sections \ref{sec:typ}-\ref{sec:efficient}. %, and the full proofs are given in the appendix.
We first reduce estimates of the magnitude of $\AS f$ to estimates on the magnitude of $\AS\activation_w$. 

\thmspace
\begin{restatable}{lemma}{lemmasingleneuron}
\label{lem:singleneuron_E}
For $f_{W,a}(x)$ given by the network \eqref{NN} with the He initialization (\cref{def:He}),
%\begin{equation}\label{m_equals_one}
$ \EE[\Xnorm{\binaryAS f_{W,a}}^2\:|W]=
   \frac{\cxh}m\sum_{k=1}^m\Xnorm{\AS\activation_{w^{(k)}}}^2$.
%\end{equation}
In particular,
\begin{equation}
    \EE[\Xnorm{\binaryAS f_{W,a}}^2]=
\cxh\:\EE[\Xnorm{\AS\activation_w}^2],\qquad w\sim\cN(0,\frac\cxh{nd}I_{nd})
\end{equation}
does not depend on $m$. 
\end{restatable}
\begin{proof}
Expand $
\Xnorm{\AS f_{W,a}}^2=\sum_{k,l=1}^ma_ka_l\bracket{\AS\activation_{w^{(k)}}}{\AS\activation_{w^{(l)}}}$. 
The $a_k\sim\cN(0,\cxh/m)$'s are independent so $\EE[a_ka_l]=\delta_{kl}\cxh/m$, and
\begin{equation}
\EE[\Xnorm{\AS f_{W,a}}^2\:|\:W]=\frac{\cxh}{m}\sum_{k,l=1}^m\delta_{kl}\EE[\bracket{\AS\activation_{w^{(k)}}}{\AS\activation_{w^{(l)}}}]=\frac{\cxh}m\sum_{k=1}^m\Xnorm{\AS\activation_{w^{(k)}}}^2.\end{equation}
Taking the expectation over $W$ yields $\EE_{W,a}[\Xnorm{\binaryAS f_{W,a}}^2]=
\cxh\:\EE[\Xnorm{\AS\activation_v}^2]
$.
\end{proof}
\thmspace
The next definition characterizes the weights of the first layer with high probability under the Xavier/He initializations.
\thmspace
\begin{definition}\label{def:typical}
Fix a constant $C>1$. We say that $w\in\RR^{nd}$ is \emph{typical} if $\cxh/2\le\|w\|^2\le2\cxh$ and $\|w\|_\infty:=\max_{ij}|w_{ij}|\le C\sqrt{\frac{\log(nd)}{nd}}$. 
\end{definition}
\thmspace
In particular a typical $w$ has $\|w\|=\Theta(1)$. We view $d$ as a constant, so $\|w\|_\infty=O(\sqrt{(\log n)/n})$ for typical $w$. For lower bounds we need an additional property of $w=w^{(k)}$ sampled from the Xavier/He initializations, namely that the $w_i\in\RR^d$ are sufficiently separated. To see why this is needed, take the example where $w_i=w_j$ for some $i\neq j$ which would imply $\AS\activation_w\equiv0$. We formalize the separation property using the following quantity:
\thmspace
\begin{definition}\label{def:delta}
For $w\in\RR^{nd}$ write $\delta_w=\frac12\min_{1\le i<j\le n}\|w_i\pm w_j\|$ where the minimum is over both choices of sign $\pm$.
\end{definition}
\thmspace
We then define weights with \emph{typical} separation:
\thmspace
\begin{definition}\label{def:typsep}
Fix a function $\delta(n)=o(n^{-(1/2+2/d)})$ as $n\to\infty$ (for concreteness we let $\delta(n)=n^{-(1/2+2/d)}/\sqrt{\log n}$).
We say that $v$ has typical separation if $\delta_v\ge\delta(n)$. We say that $W\in\RR^{m\times nd}$ is \emph{typical} if each $w^{(k)}$ is typical and at least half the $w^{(k)}$ have typical separation.
\end{definition}
\thmspace
In \cref{sec:typ} we show that He/Xavier initialized weights generically have typical separation:

\thmspace
\begin{restatable}{lemma}{WHPlemma}\label{lem:highprob}
Let $d$ be constant, let $m=O(n^{C'})$ for some constant $C'$, and let $f$ be sampled as in \cref{def:He}. Then $W$ is typical with probability $1-o(1)$ for some constant $C$ in \cref{def:typical} depending on $C'$. For such $W$, 
\begin{equation}\label{whp}
\frac{\cxh}2\inf_{w\in S'}\Xnorm{\AS\activation_w}^2\le\EE_{a|W}[\Xnorm{\AS f}^2]\le\cxh\sup_{w\in S}\Xnorm{\AS\activation_w}^2,\end{equation}
where $S$ is the set of typical $w\in\RR^{nd}$ and $S'\subset S$ is the set of typical $w$ which have typical separation. 
\end{restatable}
\thmspace

\section{Overlap kernel induced by Fourier decomposition}
Consider the case when the activation function is a complex exponential function. 
Let ``$\opn{expression}(\xfunc)$'' denote the function $x\mapsto\opn{expression}(x)$.
Then $e^{iw\cdot \xfunc}=\otimes_{i=1}^ne^{iw_i\cdot \xfunc_i}$ is a product of single-particle functions, so antisymmetrizing it yields a Slater determinant, 
\begin{equation}\label{slater}
\AS(e^{iw\cdot \xfunc})=\AS(\otimes_{i=1}^ne^{iw_i\cdot \xfunc_i})=\wedge_{i=1}^ne^{iw_i\cdot \xfunc_i},\end{equation}
where the RHS is defined as $\frac1{\sqrt{n!}}\det((e^{iw_i\cdot\xfunc_j})_{ij})$. The overlap between two Slater determinants is the determinant of the \emph{overlap matrix} \cite{PhysRev.97.1474}, meaning that 
\begin{equation}\label{AeAe}
    \bracket{\AS e^{iv\cdot\xfunc}}{\AS e^{iw\cdot\xfunc}}=\bracket{\wedge_i e^{iv_i\cdot\xfunc}}{\wedge e^{iw_i\cdot\xfunc}}=\det B^{(v,w)},
\end{equation}
where $B^{(v,w)}\in\RR^{n\times n}$ is given by
$B^{(v,w)}_{ij}=\bracket{e^{iv_i\cdot\xfunc}}{e^{iw_j\cdot\xfunc}}$. We can evaluate this as
\begin{equation}\label{defB}
B^{(v,w)}_{ij}=\EE_{X\sim\Xdistr}[e^{-i(v_i-w_j)\cdot X}]=\nnf{\Xdistr}(v_i-w_j).
\end{equation}
Here we have defined the \emph{un-normalized} Fourier transform $\nnf{}$ (also denoted by $\nnF{}{d}$) on $\RR^d$ by $\nnf{\Xdistr}(v)=\int_{\RR^d} e^{-iv\cdot x}d\Xdistr(x)$. In particular $\nnf{\Xdistr}(0)=1$, and $\nnf{\Xdistr}=(2\pi)^{d/2}\isoF\Xdistr$.

By the Fourier inversion formula \eqref{Fdef} we have the identity of functions on $\RR^{nd}$
\begin{equation}\label{FdefRnd}\activation_w(x)=\frac1{\sqrt{2\pi}}\int_{|\theta|>\epsilon}\isoF\activation(\theta)e^{i\theta w\cdot x}d\theta
+p(w\cdot x)+C_\epsilon+O(\epsilon g(w\cdot x)),\end{equation}

We use the fact that low-degree polynomials vanish upon antisymmetrization:
\thmspace
\begin{lemma}[\cite{abrahamsen2023antisymmetric} lemma 7]\label{lem:polynomial}
If $f:\RR^{nd}\to\CC$ is a polynomial of degree $\deg f\le n-2$, then $\AS f\equiv 0$. In particular $\AS\activation_w\equiv 0$ if $\activation$ is an activation function which is a polynomial of degree $\deg\activation\le n-2$.
\end{lemma}
\thmspace
%\begin{proof}
%By linearity it suffices to prove the claim when $f$ is a monomial $f(x)=\prod_{i=1}^n x_i^{r_i}$ where $r_i\in\NN_0$. Since $\deg f=\sum_i r_i\le n-2$ there exists a pair $i\neq j$ such that $r_i,r_j=0$. Let $\sigma_{ij}$ be the permutation which swaps $i$ and $j$. Then $f(\sigma_{ij}(x))=f(x)$ because $f$ does not depend on $x_i,x_j$. But we also have $f(\sigma_{ij}(x))=-f(x)$ by antisymmetry, so $f(x)=0$.
%\end{proof}

We antisymmetrize \eqref{FdefRnd} and apply \cref{lem:polynomial} which yields that for $n\ge\deg p+2$,
\begin{equation}\label{FdefA}\AS\activation_w=\lim_{\epsilon\to0}\AS\HP\activation\epsilon=\lim_{\epsilon\to0}\frac1{\sqrt{2\pi}}\int_{|\theta|>\epsilon}\isoF\activation(\theta)\AS(e^{i\theta w\cdot\xfunc})d\theta,\end{equation}
where convergence is in the $\Ltnd$-norm.
\thmspace
\begin{definition}[Overlap kernel]\label{def:Ddd}
Define $\Ddd:\RR^{nd}\times\RR^{nd}\to\CC$ by $\Ddd(v,w)=\det B^{(v,w)}$,
where $B^{(v,w)}\in\CC^{n\times n}$ is given by \eqref{defB}. %the matrix with entries $B_{ij}^{(v,w)}=\nnF{\Xdistr(v_i-w_j)}{d}$.
Given a vector of weights $w\in\RR^{nd}$ define $\Dtt{w}:\RR^2\to\CC$ by $\Dtt{w}(\theta,\tilde\theta)=\Ddd(\theta w,\tilde\theta w)$.
\end{definition}
\thmspace
Note that $\Dtt{w}$ depends on the envelope function.
Eq.~\eqref{AeAe} shows that $\Dtt{w}(\theta,\tilde\theta)=\bracket{\AS e^{i\theta w\cdot\xfunc}}{\AS e^{i\tilde\theta w\cdot\xfunc}}$. The expansion \eqref{FdefA} then yields \cref{lem:Eiint}:

\thmspace
\begin{restatable}{lemma}{lemmaEiint}\label{lem:Eiint}
If $\iint_{\RR^2} |\isoF{\activation}(\theta)\isoF{\activation}(\tilde\theta)\Dtt{w}(\theta,\tilde\theta)|d\theta d\tilde\theta<\infty$, then
\begin{equation}
\label{Eiint}
\Xnorm{\AS\activation_w}^2=\frac1{2\pi}\iint_{\RR^2}\overline{ \isoF{\activation}(\theta)}{\isoF{\activation}(\tilde\theta)}\Dtt{w}(\theta,\tilde\theta)d\theta d\tilde\theta.
\end{equation}
\end{restatable}
\thmspace

\section{Properties of the overlap kernel }
\label{sec:properties}
On the diagonal $v=w$ we have $0\le \Ddd(w,w)\le1$. Indeed, $B_{ij}^{(w,w)}$ is the Gram matrix of vectors $\AS e^{i w_i\cdot\xfunc}\in\Ltd$ so it satisfies $\det(B^{(w,w)})\le\prod_i B_{ii}^{(w,w)}=(\nnF{\Xdistr}{d}(0))^n=1$ by Hadamard's theorem for positive semidefinite matrices.
$\Ddd(v,w)$ is a positive semidefinite kernel because it is a (infinite) Gram matrix of vectors $\AS e^{iw\cdot\xfunc}\in\Ltnd$. In particular, $|\Ddd(v,w)|^2\le \Ddd(v,v)\Ddd(w,w)\le1$ for all $v,w\in\RR^{nd}$. Specializing these observations to $\Dtt{w}$ yields $0\le\Dtt{w}(\theta,\theta)\le1$ on the diagonal and $|\Dtt{w}(\theta,\tilde\theta)|^2\le1$ everywhere.

To approximate the behavior of $\Dtt{w}$ at large $\theta$ and $\tilde\theta$ we define a probability distribution $\Xdistr_w$ on $\RR$.
\thmspace
\begin{definition}\label{def:pw}
For any $w\in\RR^{nd}$, let $\Xdistr_w$ on $\RR$ be the distribution of $w\cdot X$ where $X\sim\Xdistr_n$.
\end{definition}
\thmspace

We expect $\AS e^{i\theta v\cdot\xfunc}$ to be roughly orthogonal to $\AS e^{i\tilde\theta v\cdot\xfunc}$ when the frequencies $\theta$ and $\tilde\theta$ are sufficiently different. We therefore expect $\Dtt{w}(\theta,\tilde\theta)=\bracket{\AS e^{i\theta w\cdot\xfunc}}{\AS e^{i\tilde\theta w\cdot\xfunc}}$ to vanish away from the diagonal $\tilde\theta=\theta$. Moreover, when $\theta$ is large we expect the $n!$ terms in \eqref{ASdef} to be roughly orthogonal, so by Pythagoras' theorem and considering the normalization factor $1/\sqrt{n!}$ we expect $\Dtt{w}(\theta,\theta)=\Xnorm{\AS e^{i\theta w\cdot\xfunc}}^2\approx\Xnorm{e^{i\theta w\cdot\xfunc}}^2=1$ for large $\theta$. We formalize this idea by approximating $\Dtt{w}(\theta,\tilde\theta)$ with a convolution kernel $\nnf{\rho_v}(\theta-\tilde\theta)$ when $\max\{|\theta|,|\tilde\theta|\}\ge T$ for sufficiently large $T>0$. We then apply the convolution theorem to obtain

\thmspace
\begin{restatable}{lemma}{lemmaAornot}\label{lem:highpass_A_or_not}
Let $\hp=\HP\activation{T}$ be the high-passed activation function at threshold $T>1$. Then,
\begin{equation}\label{A_or_not}
    \Xnorm{\AS\hp_w}^2=\Xnorm{\hp_w}^2+\varepsilon(T),
\end{equation}
with the error term $|\varepsilon(T)|=O(n!\delta_w^{-2}\int_{\delta_wT}^\infty\Xtail(t)tdt)$, $\Xtail(\theta)=\sup\{|\nnf{\Xdistr}(y)|:y\in\RR^d,\|y\|\ge\theta\}$.
\end{restatable}
\thmspace
\cref{lem:highpass_A_or_not} states that at sufficiently large threshold $T$ the magnitude of the antisymmetrization of the high-passed activation function is approximately the magnitude before antisymmetrization.

The $\delta_w$ in the error term was defined in \cref{def:delta}. The fast decay of $\Xtail(\theta)$ compensates for the growth of the factor $n!$ in \eqref{A_or_not}, making the approximation meaningful at polynomially large $T$.

\begin{figure}
    \centering
    \begin{minipage}[c]{.4\textwidth}
    \includegraphics[scale=.8]{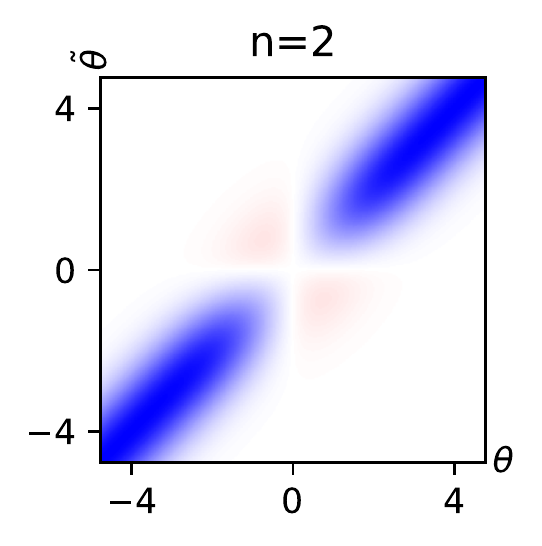}
    \end{minipage}
    %\begin{minipage}[c]{.5\textwidth}
    %\includegraphics[scale=.8]{figures/diag.pdf}
    %\end{minipage}
    \caption{
   %Left: 
   Heatmap of $\EE_w \Dtt{w}(\theta,\tilde\theta)$ where $w\sim\cN(0,\frac2{nd}I_{nd})$ for $d=3$ and $n=2$. Positive values are blue and negative values red. %Right: Diagonal cross-section $\EE_w\DttN{w}(\theta,\theta)$ for $n=100$.
   }\label{fig:Dtt}
\end{figure}

\subsection{Proof of \cref{lem:highpass_A_or_not}}
%\section{Properties of complex-exponential overlap kernel}
\label{sec:D_gen_prop}

\thmspace
\begin{restatable}{lemma}{Almostfourierlemma}\label{lem:almost_fourier_X}
Let $\delta_w=\frac12\min_{i\neq j}\|w_i\pm w_j\|$ and let $\Xdistr_w$ be as in \cref{def:pw}. Then,
\begin{align}\label{almostFourierX}
    \begin{split}
    |\Dtt{w}(\theta,\tilde\theta)-\nnf{\Xdistr_w}(\theta-\tilde\theta)|\le
n!\:\Xtail\big(\delta_w\theta\submax\big),
    \end{split}
\end{align}
    where $\theta_{\opn{max}}=\max\{|\theta|,|\tilde\theta|\}$ and $\Xtail(\theta)=\sup\{|\nnF{\Xdistr}{d}(y)|:y\in\RR^d,\|y\|\ge\theta\}$. $\nnf{\Xdistr_w}$ can be expressed in terms of the $nd$-dimensional Fourier transform as $\nnf{\Xdistr_w}(\theta)=\nnF{\Xdistr_n}{nd}(\theta w)$. 
    \end{restatable}
\thmspace
\begin{proof}
By definition we have
\[\Dtt{w}(\theta,\tilde\theta)=\det(\nnF{\Xdistr}{d}(\theta w_i-\tilde\theta w_j)_{ij})=\sum_{\pi\in S_n}(-1)^\pi\nnF{\Xdistr_n}{nd}(\theta w-\tilde\theta \pi(w)).\]
The $\pi=1$ term equals
\begin{equation}\label{nd_one}\nnF{\Xdistr_n}{nd}((\theta-\tilde\theta)w)=\nnf{\Xdistr_w}(\theta-\tilde\theta).\end{equation}
To verify the identity \eqref{nd_one} and the claim at the end of the lemma, write
\[\nnF{\Xdistr_n}{nd}(\theta w)=\EE_{X\sim\Xdistr_n}[e^{-i(\theta w)\cdot X}]=\EE_{X\sim\Xdistr_n}[e^{-i\theta (w\cdot X)}]=\EE_{y\sim\Xdistr_w}[e^{-i\theta y}]=\nnf{\Xdistr_w}(\theta).\]
The difference on the LHS of \eqref{almostFourierX} is exactly $\varepsilon=\sum_{\pi\neq 1}(-1)^\pi\nnF{\Xdistr_n}{nd}(\theta w-\tilde\theta \pi(w))$. 
Apply the triangle inequality and use the fact that $\|\Xtail\|_\infty=\|\nnF{\Xdistr_n}{nd}\|_\infty\le1$ because $\Xdistr$ is a probability distribution to obtain
\begin{align}
    \begin{split}
        |\varepsilon|
        &\le 
        n!\max_{\pi\neq1}|\nnF{\Xdistr_n}{nd}(\theta w-\tilde\theta \pi(w))|
\\&\le
n!\:\max_{\pi\neq1}\prod_{i=1}^n\Xtail\big(\|\theta w_i-\tilde\theta \pi(w)_i\|\big)
\\&\le
n!\:\max_{\pi\neq1}\min_{i}\Xtail\big(\|\theta w_i-\tilde\theta \pi(w)_i\|\big)
\le
n!\:\Xtail(|\theta|\tilde\delta(w)),
    \end{split}
\end{align}
where
    \begin{align}\label{deltainequality}\tilde\delta(v)&=\min_{\tilde\theta}\min_{\pi\neq1}\max_{i}\|v_i-\tilde\theta \pi(v)_i\|.\end{align}
We show that $\tilde\delta_w\ge\delta_w$ in \cref{lem:delta}.
 The result follows from symmetry in $\theta$ and $\tilde\theta$.
\end{proof}
    
\thmspace
\begin{lemma}\label{lem:delta}
Let $\tilde\delta_w:=\min_{\tilde\theta\in\RR}\min_{\pi\neq1}\max_{i}\|w_i-\tilde\theta \pi(w)_i\|$ Then $\tilde\delta_w\ge\delta_w$.
\end{lemma}
\thmspace
\begin{proof}
Let $\tilde\theta$ and $\sigma=\pi^{-1}\neq1$ be arbitrary. We need to show that $\| w_{i}-\tilde\theta w_{\sigma(i)}\|\ge\delta_w$ for some $i$. 

Since $\sigma\neq1$ there exists $i$ such that $\sigma(i)\neq i$, therefore $\|w_i-sw_{\sigma(i)}\|\ge2\delta_w$ for $s=\pm1$. If $\|w_{i}-\tilde\theta w_{\sigma(i)}\|\ge\delta_w$ then we have found $i$ and we are done. Otherwise, let $s=\sign\tilde\theta$. Then,
\begin{align}\label{dmax}
    \begin{split}
2\delta_w\le
\|w_{i}-sw_{\sigma(i)}\|
&\le
\|w_i-\tilde\theta w_{\sigma(i)}\|+
\|\tilde\theta w_{\sigma(i)}-sw_{\sigma(i)}\|
\\&=
\|w_i-\tilde\theta w_{\sigma(i)}\|+
|\tilde\theta-s|\cdot\|w_{\sigma(i)}\|
\\&
<\delta_w+|\tilde\theta-s|\max_j\|w_j\|.
\end{split}
\end{align}
Rearranging \eqref{dmax} we yields
\begin{equation}\label{Rdef}
    |\tilde\theta-s|R>\delta_w,\qquad R=\max_j\|w\|_j.
\end{equation}
Pick $j$ such that $\|w_j\|=R$. If $|\tilde\theta|\ge 1$ then \eqref{Rdef} says $(|\tilde\theta|-1)R>\delta_w$. Let $i=\pi(j)$ so that $\sigma(i)=j$ and $\|w_{\sigma(i)}\|=R$. Then,  
\[\|w_i-\tilde\theta w_{\sigma(i)}\|\ge \|\tilde\theta w_{\sigma(i)}\|-\|w_i\|= |\tilde\theta| R-\|w_i\|\ge(|\tilde\theta|-1)R\ge\delta_w.\]
If instead $|\tilde\theta|<1$ then \eqref{Rdef} says $(1-|\tilde\theta|)R>\delta_w$. We then use
\[\|w_j-\tilde\theta w_{\sigma(j)}\|\ge \| w_j\|-\|\tilde\theta w_{\sigma(j)}\|= R-\|\tilde\theta w_{\sigma(j)}\|\ge(1-|\tilde\theta|)R\ge\delta_w.\]
\end{proof}

%\thmspace
%\lemmaAornot*
%\thmspace
\begin{proof}[Proof of \cref{lem:highpass_A_or_not}]
By \eqref{Eiint} we have $\Xnorm{\AS\hp_w}^2=\frac1{2\pi}\iint_{|\theta|,|\tilde\theta|\ge T}\Dtt{w}(\theta,\tilde\theta)\overline{\isoF\activation(\theta)}\isoF\activation(\tilde\theta)d\theta d\tilde\theta$. Apply \cref{lem:almost_fourier_X} to approximate $\Dtt{w}(\theta,\tilde\theta)$ by $\nnf{\Xdistr_w}(\theta-\tilde\theta)$. Recall that $\isoF\hp=\isoF\activation\entp1_{\RR\del[-T,T]}$. The resulting approximation is, by the convolution theorem and Plancherel's identity,
\begin{align}\label{conv}
\begin{split}
\frac1{2\pi}\iint\nnf{\Xdistr_w}(\theta-\tilde\theta)\overline{\isoF{\hp}(\theta)}\isoF{\hp}(\tilde\theta)d\theta d\tilde\theta.
=&
\frac1{2\pi}\int\overline{\isoF{\hp}(\theta)}\Big(\nnf{\Xdistr_w}\conv\isoF{\hp}\Big)(\theta)d\theta.
\\=&
\int|\hp(y)|^2\Xdistr_w(y)dy.
=\EE_{X\sim\Xdistr_n}[|h(w\cdot X)|^2]
=\Xnorm{\hp_w}^2.
\end{split}
\end{align}
$\isoF\activation$ is bounded on $\RR\del[-T,T]$ by assumption. The error of the approximation is then bounded by a constant times
\begin{align*}
&\iint_{|\theta|,|\tilde\theta|\ge T}|\Dtt{w}(\theta,\tilde\theta)-\nnf{\Xdistr_w}(\theta-\tilde\theta)|d\theta d\tilde\theta
\le
8\int_T^\infty\int_{T}^{\theta}n!\Xtail\big(\delta_w\theta\big)d\tilde\theta d\theta 
\le 8n!\int_T^\infty \Xtail\big(\delta_w\theta\big)\:\theta d\theta.
\end{align*}
The error bound follows by substituting $t=\delta_w\theta$.
\end{proof}

\subsection{Upper bound for Gaussian envelopes}

\cref{lem:highpass_A_or_not} explains the behavior of $\Dtt{w}(\theta,\tilde\theta)$ at large $\theta$. We now establish that it vanishes for small $\theta$.
When the envelope is the standard Gaussian $\Xdistr=\cN$, the overlap kernel takes the following form: 
\begin{equation}\label{GaussianD}
\DddN(v,w)=e^{-\frac{\|v\|^2+\|w\|^2}2}\det\big((e^{v_i\cdot w_j})_{ij}\big).
\end{equation}
An upper bound on $\DddN$ was obtained in \cite[Proposition 11]{abrahamsen2023antisymmetric}:

\thmspace
\begin{restatable}{proposition}{propdetbound}\label{prop:detbound}
Let $\nu=2\sqrt{d\|v\|_\infty\|w\|_\infty}$. Then,
$\det((e^{v_i\cdot w_j})_{ij})\le (\nu/2)^{pn}$ for $\nu\le1$, where $p$ is any integer such that $\binom{p+d-1}{d}\le n/2$ and $p!\ge 4n^2$.
\end{restatable}
\thmspace

The proof of proposition \ref{prop:detbound} is recalled from \cite{abrahamsen2023antisymmetric} in section \ref{sec:det_upper_bound}. 
It works by decomposing $(e^{v_i\cdot w_j})_{ij}=\sum_{k=0}^\infty Q_k$ and bounding the ranks and operator norms of the terms $Q_k$. For $L=\sum_{k=1}^{p-1}\rank Q_k$, the $L$-th eigenvalue is then bounded as the tail sum $\sum_{k=p}^\infty\|Q_k\|$. Taking the product of the eigenvalues yields the bound on \eqref{GaussianD}.

Combining proposition \ref{prop:detbound} with the triangle inequality yields:
\thmspace
\begin{restatable}{lemma}{lemmalpbound}\label{lem:lpbound}
Let $\lp=\LP\activation{t}$ be the low-pass at threshold $t=(2\sqrt d\|w\|_\infty)^{-1}$. 
If $w$ is typical then $t=\Omega(\sqrt{n/\log n})$ and $\XnormN{\AS\lp_w}=O(2^{-\Omega(n^{1+1/d})})$.
\end{restatable}
\thmspace
%

%The high-pass part can be estimated using only the triangle inequality and the $\ell^1$-decay of $\isoF\activation$, but for a sharper bound in terms of the $\ell^2$-decay we use \eqref{GaussianD} to obtain the relation
%\begin{equation}
%    \label{depends_on_product}
%    \DttN{w}(\theta,\tilde\theta)=e^{-\frac{\|w\|^2}{2}(|\theta|-|\tilde\theta|)^2}\DttN{w}(\theta\sub{g.m.},\pm\theta\sub{g.m.}),
%\end{equation}
%where the geometric mean $\theta\sub{g.m.}={|\theta\tilde\theta|}^{1/2}$ and $\pm$ is the sign of $\theta\tilde\theta$. In particular $\DttN{w}(\theta,\tilde\theta)\le e^{-\frac12{\|w\|^2}(|\theta|-|\tilde\theta|)^2}$. We then apply \eqref{Eiint} and Cauchy-Schwarz to obtain

\section{Proof of \cref{thm:smooththm}}
We specialize the quantities of \cref{lem:highpass_A_or_not} to the case of a Gaussian envelope.
For the Gaussian envelope $\Xdistr=\cN(0,I_d)$ we have $\nnF{\Xdistr}{d}(v)=e^{-\|v\|^2/2}$, $\nnf{\Xdistr_w}(\theta)=e^{-\|w\|^2\theta^2/2}$, and $\Xtail(\theta)=e^{-\theta^2/2}$. The integral in the error term of \cref{lem:highpass_A_or_not} becomes $\int_{\delta_wT}^\infty e^{-t^2/2}tdt=e^{-(\delta_wT)^2/2}$, so we get the approximation
\begin{equation}\label{A_or_not_gaussian_case}
    \XnormN{\AS\hp_w}^2=\XnormN{\hp_w}^2+O(\delta_w^{-2}e^{-\delta_w^2T^2/2+n\log n}).
\end{equation}

We have the two following expressions for $\DttN{w}$.
\begin{align}
        \DttN{w}(\theta,\tilde\theta)
        &=
        \det[(e^{-\frac12\|\theta w_i-\tilde\theta w_j\|^2})_{ij}]
        =
        e^{-\frac{\theta^2+\tilde\theta^2}{2}\|w\|^2}\det[(e^{\theta\tilde\theta w_i\cdot w_j})_{ij}].
        \label{D_e_det}
\end{align}

It follows from \eqref{D_e_det} that $\DttN{w}(\theta,\tilde\theta)$ can be determined from its values on the diagonal and anti-diagonal $\tilde\theta=\pm\theta$. More precisely we have \eqref{depends_on_product},

\begin{equation}
    \label{depends_on_product}
    \DttN{w}(\theta,\tilde\theta)=e^{-\frac{\|w\|^2}{2}(|\theta|-|\tilde\theta|)^2}\DttN{w}(\theta\sub{g.m.},\pm\theta\sub{g.m.}),
\end{equation}
%\begin{equation*}
%    \DttN{w}(\theta,\tilde\theta)=e^{-\frac{\|w\|^2}{2}(|\theta|-|\tilde\theta|)^2}\DttN{w}(\theta\sub{g.m.},\pm\theta\sub{g.m.}),
%\end{equation*}
where the geometric mean $\theta\sub{g.m.}={|\theta\tilde\theta|}^{1/2}$ and $\pm$ is the sign of $\theta\tilde\theta$. 
To show \eqref{depends_on_product} apply the rightmost expression of \eqref{D_e_det} on both sides and note that the determinants are equal. In particular we have that $\DttN{w}$ decays away from the diagonal and anti-diagonal,
\begin{equation}\label{diagonalconcentration}
    |\DttN{w}(\theta,\tilde\theta)|\le e^{-\frac{\|w\|^2}{2}(|\theta|-|\tilde\theta|)^2}.
\end{equation}
The bound \eqref{diagonalconcentration} gives concentration around the diagonal everywhere and not only for large $\theta,\tilde\theta$ as \cref{lem:almost_fourier_X}. Integrating \eqref{diagonalconcentration} yields:

%\thmspace
%\lemmahpbound*
%\thmspace
\thmspace
\begin{restatable}{lemma}{lemmahpbound}\label{lem:hpbound}
Let $\hp=\HP\activation{t}$ be the high-passed activation function at $t>1$. Then,
$\XnormN{\AS\hp_w}^2
    \le\frac4{\sqrt{2\pi}\|w\|}
    \tailsum\activation(t)$. So $\XnormN{\AS\hp_w}^2=O(t^{-K})$ for typical $w$.
\end{restatable}
\thmspace
\begin{proof}
Write the LHS as a double integral over $|\theta|,|\tilde\theta|\ge t$ as in \eqref{Eiint}. By \eqref{diagonalconcentration},
\begin{align}
   \frac1{2\pi}\iint_{\tilde\theta>\theta>t}\DttN{w}(\theta,\tilde\theta)|\isoF\activation(\theta)\isoF\activation(\tilde\theta)|d\theta d\tilde\theta
   &=
   \frac1{2\pi}\int_0^\infty e^{-\frac{\|w\|^2s^2}{2}}\int_t^\infty |\isoF\activation(\theta)\isoF\activation(s+\theta)|d\theta ds
   \\&\le 
   \frac1{2\pi}\int_0^\infty e^{-\frac{\|w\|^2s^2}{2}} ds\int_t^\infty |\isoF\activation(\theta)|^2d\theta\label{applyCS}
   \\&=\frac{1}{2\sqrt{2\pi}\|w\|}\int_t^\infty |\isoF\activation(\theta)|^2d\theta,
\end{align}
where \eqref{applyCS} follows from Cauchy-Schwartz.
The same bound holds for each of $8$ regions in the $\theta,\tilde\theta$-plane.
\end{proof}

%\thmspace
%\begin{restatable}{lemma}{lemmahpbound}\label{lem:hpbound}
%Let $\hp=\HP\activation{t}$ be the high-passed activation function at $t>1$. Then,
%$\XnormN{\AS\hp_w}^2
%    \le\frac4{\sqrt{2\pi}\|w\|}
%    \tailsum\activation(t)$. So $\XnormN{\AS\hp_w}^2=O(t^{-K})$ for typical $w$.
%\end{restatable}
%\thmspace

For typical $w$ choose threshold $t=\Omega(\sqrt{n/\log n})$ as in \cref{lem:lpbound} and decompose $\AS\activation_w=\AS\lp_w+\AS\hp_w$. Apply \cref{lem:lpbound} and \cref{lem:hpbound} to obtain the following corollary, which implies \cref{thm:smooththm}.

\thmspace
\begin{corollary}\label{cor:smooth}
If $\activation$ has frequency tail decay $K<\infty$ then $\XnormN{\AS\activation_w}^2=\tilde O(n^{-K/2})$ for typical $w$. If $\activation$ is smooth ($K=\infty$), then $\XnormN{\AS\activation_w}^2=n^{-\omega(1)}$ for typical $w$.
\end{corollary}
\thmspace

\section{Proof of \cref{thm:roughthm}}
\label{sec:slowdecayproof}
%\subsection{Proof of lower bound}
%\label{sec:lowerbound}
%For the Gaussian envelope function the error term in \cref{lem:highpass_A_or_not} becomes $O(\delta_w^{-2}e^{-\delta_w^2T^2/2+n\log n})$. To prove the lower bound we pick $T=\frac2{\delta_w}(\sqrt{n\log n+\log(1/\delta_w)})$ and  decompose $\activation$ into $\lp=\LP\activation{T}$ and $\hp=\HP\activation{T}$. We then expand $\XnormN{\AS\lp+\AS\hp}^2$ and use that $\XnormN{\AS\lp}^2\ge0$ to get $\XnormN{\AS\activation_w}^2\ge\XnormN{\AS\hp_w}^2-2|\bracketN{\AS\lp_w}{\AS\hp_w}|$. The result follows by bounding $|\bracketN{\AS\lp_w}{\AS\hp_w}|$  and applying \cref{lem:highpass_A_or_not} to get $\XnormN{\AS\hp_w}^2\ge\XnormN{\hp_w}-O(e^{-n\log n})$. 
%We then use the convolution theorem to show that under our definition of roughness (or generalized roughness), $\XnormN{\hp_w}^2=\Omega(\tailsum{\activation}(T))=\Omega(T^{-K})$. There is a gap between the upper bound $O(t^{-K})=\tilde O(n^{-K/2})$ and this lower bound $\Omega(T^{-K})=\tilde\Omega(n^{-(1+2/d)K})$ due to the factor $1/\delta_w$ in the threshold $T$ compared to $t$.
%

For the Gaussian envelope function we get a more explicit form of the error term in \cref{lem:highpass_A_or_not}. To prove the lower bound we pick an appropriate $T$ based on this expression and  decompose $\activation$ into $\lp=\LP\activation{T}$ and $\hp=\HP\activation{T}$. We then expand $\XnormN{\AS\lp+\AS\hp}^2$ and use that $\XnormN{\AS\lp}^2\ge0$ to get $\XnormN{\AS\activation_w}^2\ge\XnormN{\AS\hp_w}^2-2|\bracketN{\AS\lp_w}{\AS\hp_w}|$. The result follows by bounding $|\bracketN{\AS\lp_w}{\AS\hp_w}|$  and applying \cref{lem:highpass_A_or_not} to lower-bound $\XnormN{\AS\hp_w}$.

\thmspace
\begin{lemma}\label{lowerlemma}
For activation functions satisfying \cref{it:boundpointbytail} of \cref{def:generalrough} we have that
\begin{equation}\label{thres}
\XnormN{\AS\activation_w}^2\ge\XnormN{\hpf{T}_w}^2-O(e^{-n\log n})+o(\tailsum\activation(T)).\end{equation}
\end{lemma}
\thmspace
for $T=\frac2{\delta_v}(\sqrt{n\log n+\log(1/\delta_v)})$ and typical $w$. 
\begin{proof}
Let $\lpf{T}=\LP\activation{T}$ and $\hpf{T}=\HP\activation{T}$.
As discussed above we have the lower bound
\begin{equation}\label{minuscrossterms}
\XnormN{\AS\activation_w}^2\ge\XnormN{\AS\hpf{T}_w}^2-2|\bracketN{\AS\lpf{T}_w}{\AS\hpf{T}}|\ge\XnormN{\hpf{T}_w}^2-e^{-n\log n}-2|\bracketN{\AS\lpf{T}_w}{\AS\hpf{T}_w}|,\end{equation}
We further divide the low-frequency part into $|\theta|$ in $[0,1]$, $[1,T/2]$, and $[T/2,T]$. We write $\alpha^{(t,T)}=\hpf{t}-\hpf{T}$. Then,
\begin{equation}\label{crossbound}
    |\bracketN{\AS\lpf{T}_w}{\AS\hpf{T}_w}|\le\XnormN{\AS\lpf{1}_w}\XnormN{\AS\hpf{T}_w}+|\bracketN{\AS\alpha^{(1,T/2)}_w}{\AS\hpf{T}_w}+\bracketN{\AS\alpha^{(T/2,T)}_w}{\AS\hpf{T}_w}|.
\end{equation}
By applying the polarization identity to \eqref{Eiint} we obtain the overlap between the antisymmetrization with different activation functions, so by the bound \eqref{diagonalconcentration} on $\DttN{w}(\theta,\tilde\theta)$,  
\begin{align}
|\bracketN{\AS\alpha^{(T/2,T)}_w}{\AS\hpf{T}_w}|
&\le
\frac{(\sup_{|\theta|\ge T/2}|\isoF\activation(\theta)|)^2}{2\pi}\int_{T/2<|\theta|<T}\int_{T<|\tilde\theta|}
e^{-\frac{\|w\|^2}2(|\tilde\theta|-|\theta|)^2}d\tilde\theta d\theta
\\
&\le
\frac{4(\sup_{|\theta|\ge T/2}|\isoF\activation(\theta)|)^2}{2\pi}\int_{T/2<\theta<T}\int_{T-\theta}^{\infty}
e^{-\frac{\|w\|^2}2t^2}dtd\theta\\
&=
\frac{4(\sup_{|\theta|\ge T/2}|\isoF\activation(\theta)|)^2}{2\pi}\int_{0}^{\infty}
\min\{t,T/2\}e^{-\frac{\|w\|^2}2t^2}dt
\\
&\le
\frac{4\sup_{|\theta|\ge T/2}|\isoF\activation(\theta)|^2}{2\pi}\int_0^\infty  te^{-\frac{\|w\|^2t^2}{2}}dt\label{subt}
\\&=\frac{2\sup_{|\theta|\ge T/2}|\isoF\activation(\theta)|^2}{\pi\|w\|^2},
%\label{crossterm3}
\end{align}
where in \eqref{subt} the factor $4$ comes from the choice of signs of $\theta,\tilde{\theta}$, and we have substituted $t=\tilde\theta-\theta$.
We apply \cref{it:boundpointbytail} of \cref{def:generalrough} to obtain that
\begin{equation}\label{crossterm3}
|\bracketN{\AS\alpha^{(T/2,T)}_w}{\AS\hpf{T}_w}|=o(\tailsum{\activation}(T)/\|w\|^2).
\end{equation}

For the first term on the RHS of \eqref{crossbound} we have for $T$ sufficiently large (so that $|\isoF\activation(\theta)\isoF\activation(\tilde\theta)|\le2\pi$ for $|\theta|>1$ and $|\tilde\theta|>T$),
\begin{align}
\begin{split}
|\bracketN{\AS\alpha^{(1,T/2)}_w}{\AS\hpf{T}_w}|
&\le
\int_{1<|\theta|<T/2}\int_{T<|\tilde\theta|}
e^{-\frac{\|w\|^2}2(|\tilde\theta|-|\theta|)^2}d\tilde\theta d\theta
\\&\le
\frac{T}2\int_{T/2}^\infty e^{-\frac{\|w\|^2t^2}2}dt
\\&\le
\frac{e^{-\frac{\|w\|^2T^2}8}}{\|w\|^2}.
\label{crossterm2}
\end{split}
\end{align}
By \eqref{crossterm3} and \eqref{crossterm2}, $|\bracketN{\alpha^{(1,T/2)}_w}{\hpf{T}_w}+\bracketN{\alpha^{(T/2,T)}_w}{\hpf{T}_w}|=O(e^{-\Omega(T^2)})+o(\tailsum\activation(T))$ for typical $w$. Finally, for typical $w$ we have $t>1$ in \cref{lem:lpbound} and $\XnormN{\hpf{T}_w}=O(t^{-K})=O(1)$ by \cref{lem:hpbound}, so $\XnormN{\lpf{1}_w}\XnormN{\hpf{T}_w}=O(\XnormN{\lpf{t}_w})=2^{-\Omega(n^{1+1/d})}$. The claim follows by substituting back into \eqref{minuscrossterms} and \eqref{crossbound}.
\end{proof}

\begin{proof}[Proof of \cref{thm:roughthm}]
By \cref{lem:highprob} it suffices to show the claim for $f=\activation_w$ where $w$ is typical and has typical separation.

By the convolution theorem \eqref{conv} and the expression $\nnf{\Xdistr_v}(\theta)=e^{-\|v\|^2\theta^2/2}$ for the Gaussian case we have
\begin{equation}\label{h_iint}
\XnormN{\hpf{T}_w}^2
=\frac1{2\pi}\iint_{|\theta|,|\tilde\theta|>T}e^{-\|w\|^2(\theta-\tilde\theta)^2/2}\overline{\isoF\activation(\theta)}\isoF\activation(\tilde\theta)d\theta d\tilde\theta.
\end{equation}
We bound the contribution of the first quadrant to \eqref{h_iint} from below. The same argument holds for the third quadrant ($\theta,\tilde\theta<-T$). Let $\sigma=1/\|w\|$. Then,
\begin{align}
\begin{split}
\iint_{\theta,\tilde\theta>T}e^{-\|w\|^2(\theta-\tilde\theta)^2/2}\overline{\isoF\activation(\theta)}\isoF\activation(\tilde\theta)d\theta d\tilde\theta
&=
2\Real\int_T^\infty\overline{\isoF\activation(\theta)}\int_\theta^\infty e^{-\|w\|^2(\tilde\theta-\theta)^2/2}\isoF\activation(\tilde\theta) d\tilde\theta d\theta
\\&=
2\Real\int_T^\infty\overline{\isoF\activation(\theta)}\int_\theta^\infty e^{-\frac{(\tilde\theta-\theta)^2}{2\sigma^2}}\isoF\activation(\tilde\theta) d\tilde\theta d\theta
\\&=
\sqrt{2\pi}\sigma\int_T^\infty\Real\Big(\overline{\isoF\activation(\theta)}\EE[\isoF\activation(\theta+|Y|)] \Big)d\theta,
\end{split}\end{align}
where $Y\sim\cN(0,\sigma^2)$. 
To obtain a lower bound on the integrand we write
\begin{equation}
    \overline{\isoF\activation(\theta)}\cdot\EE[\isoF\activation(\theta+|Y|)]
    =
    \|\isoF\activation(\theta)\|^2\cdot\frac{\EE[\isoF\activation(\theta+|Y|)]}{\isoF\activation(\theta)}.
\end{equation}
So from \cref{it:previouslyrealcondition} of \cref{def:generalrough} we then have a lower bound
\begin{equation}\label{applypreviouslyreal}
    \Real\Big(\overline{\isoF\activation(\theta)}\cdot\EE[\isoF\activation(\theta+|Y|)]\Big)
    =
    \|\isoF\activation(\theta)\|^2\Real\Big(\frac{\EE[\isoF\activation(\theta+|Y|)]}{\isoF\activation(\theta)}\Big)=\Omega(|\isoF\activation(\theta)|^2),
\end{equation}
when $1/2\le\sigma\le2$, i.e., when $1/4\le\|w\|^2\le4$. This holds because $w$ is typical (and because $1/2\le\cxh\le2$). For typical $w$ \eqref{h_iint} and \eqref{applypreviouslyreal} then show that
\begin{align}\label{diagpart}
\begin{split}
\iint_{\theta,\tilde\theta>T}e^{-\|w\|^2(\theta-\tilde\theta)^2/2}\overline{\isoF\activation(\theta)}\isoF\activation(\tilde\theta)d\theta d\tilde\theta
&=
\Omega\Big(\int_T^\infty|\isoF\activation(\theta)|^2d\theta\Big)=\Omega(\tailsum\activation(T)),
\end{split}\end{align}
and the same lower bound holds for the integral over $\theta,\tilde\theta<T$. Finally we bound the second quadrant $\theta<-T,\tilde\theta>t$ (the fourth quadrant $\theta>T,\tilde\theta<-T$ is analogous) by writing
\begin{align}\label{offdiag}
\begin{split}
\Big|\iint_{\theta<-T,\tilde\theta>T}e^{-\|w\|^2(\theta-\tilde\theta)^2/2}\overline{\isoF\activation(\theta)}\isoF\activation(\tilde\theta)d\theta d\tilde\theta\Big|
&=
o\Big(\iint_{\theta<-T,\tilde\theta>T}e^{-\|w\|^2(\tilde\theta-\theta)^2/2}d\theta d\tilde\theta\Big)
\\&=
o\Big(\int_{2T}^\infty te^{-\|w\|^2t^2/2}dt\Big)
\\&=
o(e^{-2\|w\|^2T^2}/\|w\|^2)=o(e^{-T^2}),
\end{split}\end{align}
where the last expression os for typical $w$. The identity \eqref{h_iint}, the lower bound for the diagonal part \eqref{diagpart}, and the bound on the magnitude of the off-diagonal part \eqref{offdiag} show that $\XnormN{\hpf{T}_w}^2=\Omega(\tailsum\activation(T))=\Omega(T^{-K})$. For typical $w$ with typical separation we have that $\delta_w\ge n^{-(1/2+2/d)}/\sqrt{\log n}$ and $T\sim2n^{1+2/d}\log n$, so $\XnormN{\hpf{T}_w}^2=\tilde\Omega(n^{-(1+2/d)})$.
\end{proof}

\section{Efficient algorithm}
Recall \eqref{slater} which gives the antisymmetrization with an exponential activation function as a determinant, $\AS(e^{iw\cdot \xfunc})=\frac1{\sqrt{n!}}\det((e^{iw_i\cdot\xfunc_j})_{ij})$. Let $\lpf{t}=\LP\activation{t}$ and $\hpf{T}=\HP\activation{T}$ be the low-pass and high-pass of $\activation$ at thresholds $t<T$. We approximate $\AS\activation_w$ by removing the low-passed and high-passed components: Apply \eqref{FdefA} to $\HP\activation{t}-\HP\activation{T}$ to obtain $\AS\activation_w(x)
=\alpha_w+\AS\lpf{t}_w+\AS\hpf{T}_w$, where
\begin{align}\label{cutoff}
\alpha_w(x)&=\frac1{\sqrt{2\pi n!}}\int_{[-T,T]\del[-t,t]} \isoF{\activation}(\theta)\det((e^{i\theta w_i\cdot x_j})_{ij})d\theta.
\end{align}
The integrand can be computed at a single $\theta$ in time $O(n^3)$. Apply \cref{lem:lpbound} and \cref{lem:hpbound} (with a different choice of threshold for the high-pass) to bound the truncation error.

\thmspace
\begin{lemma}[Truncation error bound]
\label{lem:truncerror}
Suppose $\isoF\activation$ has tail decay $K$. Let  $t=\max\{\frac1{2\sqrt{d}\|w\|_\infty},1\}$ and, given $\epsilon>0$, let $T=\epsilon^{-1/K}$. 
Then 
$\XnormN{\AS\activation_w-\alpha_w}^2= O(2^{-\Omega(n^{1+1/d})}+\epsilon)$ for typical $w$. 
\end{lemma}
\thmspace
We approximate $\alpha_w(x)$ by a sum $\label{discsum}S_w(x)=\frac1{\sqrt{2\pi n!}}\sum_{p=\pm1,\ldots,\pm N}c_p\det((e^{i\theta_pw_i\cdot x_j})_{ij})$,
where $\theta_1,\ldots,\theta_N$ are a discretization of $[t,T]$. 

\subsection{Discretization error bound}
\label{sec:efficient}

Let $t=\Omega(1)$ and let $t=t_0<t_1<\cdots<t_N=T$ be evenly spaced and for each $p=1,\ldots,N$ write  \[I_{-p}=[-t_p,-t_{p-1}],\qquad I_p=[t_{p-1},t_p].\]
For $p=\pm1,\ldots,\pm N$, let $\theta_p\in I_p$ and define \[c_p:=\int_{I_p}\isoF\activation(\theta)d\theta.\]

For these evenly spaced $\theta_p$ and coefficients $c_p$, define $\|\partial_\theta\DttN{w}\|_\infty:=\sup_{\theta,\tilde\theta}|\frac{\partial}{\partial\tilde\theta}\DttN{w}(\theta,\tilde\theta)|$. We then have
\thmspace
\begin{restatable}{lemma}{lemmadiscerror}
\label{lem:discerror}
$\XnormN{S_w-\alpha_w}^2
\le
\frac2\pi\frac TN\|\partial_\theta\DttN{w}\|_\infty\big(\int_{t\le|\theta|\le T}|\isoF\activation(\theta)|\big)^2=O(\frac{T}N\|\partial_\theta\DttN{w}\|_\infty)$.
\end{restatable}
%\thmspace
%\lemmadiscerror*
%\thmspace
\begin{proof}
$S_w$
is exactly the antisymmetrization $\AS s_w$ where
\[s(t)=\frac1{\sqrt{2\pi}}\sum_{\pm q=1}^Nc_qe^{i\theta_qt}.\]
Passing from the truncated activation to its discretization incurs the error
\begin{equation}\XnormN{S_w-\tilde\alpha_w}^2=\XnormN{\AS(s-\activation\sub{trc})_w}^2=\frac1{2\pi}\iint_{\RR^2}\DttN{w}(\theta,\tilde\theta)d\bar\mu(\theta)d\mu(\tilde\theta),\label{signediint}\end{equation}
where $\activation\sub{trc}=\activation-\LP{\activation}{1}-\HP{\activation}{T}$, and $\mu$ is the complex-valued measure  
\[\mu(\theta)=\sum_qc_q\delta(\theta-\theta_q)-(\isoF\activation(\theta)d\theta).\]
For each square $S_{pq}=I_p\times I_q$,  
\begin{align}\label{boundonsquare}
\begin{split}
&\iint_{S_{pq}}\DttN{w}(\theta,\tilde\theta)d\bar\mu(\theta)d\mu(\tilde\theta)
\\[2em]=&
\overline{c_p}c_q\DttN{w}(\theta_p,\theta_q)+\iint_{S_{pq}}\DttN{w}(\theta,\tilde\theta)\overline{\isoF\activation(\theta)}\isoF\activation(\tilde\theta)d\theta d\tilde\theta
\\&-\overline{c_p}\int_{I_q}\DttN{w}(\theta_p,\tilde\theta)\isoF\activation(\tilde\theta)d\theta
-c_q\int_{I_p}\DttN{w}(\theta,\theta_q)\overline{\isoF\activation(\theta)}d\theta
\\[2em]=&
\iint_{S_{pq}}\Big[\DttN{w}(\theta_p,\theta_q)+\DttN{w}(\theta,\tilde\theta)-\DttN{w}(\theta_p,\tilde\theta)-\DttN{w}(\theta,\theta_q)\Big]\:\overline{\isoF\activation(\theta)}\isoF\activation(\tilde\theta)d\theta d\tilde\theta
\\[2em]\le&
\iint_{S_{pq}}\Big[2\max \DttN{w}(I_p,I_q)-2\min \DttN{w}(I_p,I_q)\Big]\:|\isoF\activation(\theta)\isoF\activation(\tilde\theta)|d\theta d\tilde\theta
\\[2em]=&2(\opn{diam} \DttN{w}(I_p,I_q))\int_{I_p}|\isoF\activation(\theta)|d\theta\int_{I_q}|\isoF\activation(\tilde\theta)|d\tilde\theta,
\end{split}
\end{align}
where $\opn{diam}(\DttN{w}(I_p,I_q))$ is the diameter of the set $D(I_p,I_q)=\{\DttN{w}(\theta,\tilde\theta)\:|\:(\theta,\tilde\theta)\in I_p\times I_q\}$.
Let $\|\partial \DttN{w}\|_\infty=\sup_{\theta,\tilde\theta}\max\{|\frac{\partial}{\partial\theta}\DttN{w}(\theta,\tilde\theta)|,|\frac{\partial}{\partial\tilde\theta}\DttN{w}(\theta,\tilde\theta)|\}$ and let $t_0=1,\ldots,t_N=T$ be evenly spaced so that $|I_p|\le T/N$. 
Then by \eqref{boundonsquare},
\begin{align}
\iint_{S_{pq}}\DttN{w}(\theta,\tilde\theta)d\mu(\theta)d\mu(\tilde\theta)
\le&
\frac{4T}N\|\partial \DttN{w}\|_\infty\int_{I_p}|\isoF\activation(\theta)|d\theta\int_{I_q}|\isoF\activation(\tilde\theta)|d\tilde\theta.
\end{align}
Summing over $p,q$ and applying \eqref{signediint} we get
\[\XnormN{S_w-\tilde\alpha_w}^2
\le
\frac2\pi\frac TN\|\partial \DttN{w}\|_\infty\Big(\int_{t\le|\theta|\le T}|\isoF\activation(\theta)|\Big)^2.\]
\end{proof}

\thmspace
\begin{lemma}\label{lem:gradbound}
Let $B^{(v,w)}_{ij}=\nnf{\Xdistr}(v_i-w_j)$ as in \cref{def:Ddd}. Then, $|\frac\partial{\partial B_{ij}}\det B|\le 1$ at any $B=B^{(v,w)}$.
\end{lemma}
\thmspace
\begin{proof}
$\frac\partial{\partial B_{ij}}\det B=(-1)^{i+j}m^B_{i,j}$ where $m^B_{i,j}$ is the $i,j$th minor of $B$.  
But $m^B_{ij}=\det B^{(\tilde v,\tilde w)}$ where  $\tilde v=(v_{i'})_{i'\neq i}$ and $\tilde w=(w_{j'})_{j'\neq j}$, and $|\det B^{(\tilde v,\tilde w)}|=|\DddN(\tilde v,\tilde w)|\le 1$ by the properties mentioned in \cref{sec:D_gen_prop}. 
\end{proof}
\begin{corollary}\label{cor:Dgradbound}
$|\frac\partial{\partial\theta}\DttN{w}(\theta,\tilde\theta)|\le n^{3/2}\|w\|/\sqrt e$.
\end{corollary}
\begin{proof}
By the chain rule,
\[\frac\partial{\partial\theta}\DttN{w}(\theta,\tilde\theta)=
\sum_{ij}\frac{\partial B_{ij}(\theta w,\tilde\theta w)}{\partial\theta}\frac{\partial\det B}{\partial B_{ij}},\]
where
\[\Big|\frac{\partial B_{ij}(\theta w,\tilde\theta w)}{\partial\theta}\Big|=|w_i\cdot(\nabla\nnF{\Xdistr}{d})(\theta w_i-\tilde\theta w_j)|\le \|w_i\|\cdot\|\nabla\nnF{\Xdistr}{d}\|_\infty\]
\cref{lem:gradbound} and the triangle inequality then imply that
\[\Big|\frac\partial{\partial\theta}\DttN{w}(\theta,\tilde\theta)\Big|\le n\|\nabla\nnF{\Xdistr}{d}\|_\infty\sum_i\|w_i\|\le n^{3/2}\|\nabla\nnF{\Xdistr}{d}\|_\infty\|w\|,\]
where the last inequality is by Cauchy-Schwarz. For $\Xdistr=\cN(0,I_d)$ we have $\nnF{\cN}{d}(w)=e^{-\|w\|^2/2}$ and $\|\nabla\nnF{\cN}{d}\|_\infty=1/\sqrt e$.
\end{proof}

%\thmspace
%\begin{restatable}{lemma}{lemmadiscerror}
%\label{lem:discerror}
%$\XnormN{S_w-\alpha_w}^2
%\le
%\frac2\pi\frac TN\|\partial_\theta\DttN{w}\|_\infty\big(\int_{t\le|\theta|\le T}|\isoF\activation(\theta)|\big)^2=O(\frac{T}N\|\partial_\theta\DttN{w}\|_\infty)$.
%\end{restatable}
%\thmspace
In the presence of bias terms $b_k$ in \eqref{NN} we can efficiently compute an approximation $S_{w,b}$ to the function $\AS\activation(w^{(k)}\cdot\xfunc+b_k)$ with the same error bound. Indeed, the bias term results in a shift of the activation function which corresponds to multiplying the Fourier transform by an oscillating phase. Since the upper bounds do not depend on the phase of the Fourier transform, the same  truncation error bound applies. \cref{thm:smooththm} holds for arbitrary bias terms for the same reason.
\begin{proof}[Proof of \cref{thm:efficient_multiplicative_error}]
Because $\activation$ is rough we have $\EE[\XnormN{\AS f}|W]=n^{-O(1)}$ with probability $1-o(1)$. Given target relative error $\epsilon=n^{-O(1)}$ it suffices to achieve absolute error $\epsilon'=\epsilon\cdot\EE[\XnormN{\AS f}|W]=n^{-O(1)}$. $\XnormN{S_w-\AS\activation_w}^2=O(\epsilon'+n^{O(1)}\|\partial_\theta\DttN{w}\|_\infty/N)$ by \cref{lem:truncerror} and \cref{lem:discerror}. \cref{cor:Dgradbound} shows that $\|\partial_\theta\DttN{w}\|_\infty=n^{O(1)}$, so $\XnormN{S_w-\AS\activation_w}^2=O(\epsilon'+n^{O(1)}/N)$. It then suffices to pick $N=n^{O(1)}$. Let $S_{W,a,b}=\sum_{k=1}^m a_kS_{w^{(k)},b_k}$ and apply \cref{lem:singleneuron_E} to the difference $f_{W,a,b}-S_{W,a,b}$ to extend the termwise error bound to the sum \eqref{NN}. The computational cost of evaluating $S_w(x)$ is $O(n^3N)$.
\end{proof}

%\rev{

\begin{revs}
    
\subsection{Numerical demonstration of Theorem \ref{thm:efficient_multiplicative_error}}

We numerically demonstrate Theorem \ref{thm:efficient_multiplicative_error} by approximating the anti-symmetrization of a single neuron with the ReLU activation function (Figures \ref{fig:err1} and \ref{fig:err2}). We compare the approximation $S_w$ given by Theorem \ref{thm:efficient_multiplicative_error} with the explicit anti-symmetrization $\mathcal A\tau_w$. We use 100 sample points $x\sim\mathcal N(0,I)$ to estimate the norm of the anti-symmetrized function and the distance between the true anti-symmetrization and its efficient approximation. In our implementation we used Gauss-Legendre quadrature to estimate the integral in \eqref{cutoff}.
  
\begin{figure}[h]
    \centering
    \begin{minipage}[c]{.9\textwidth}
    \includegraphics[width=.9\textwidth]{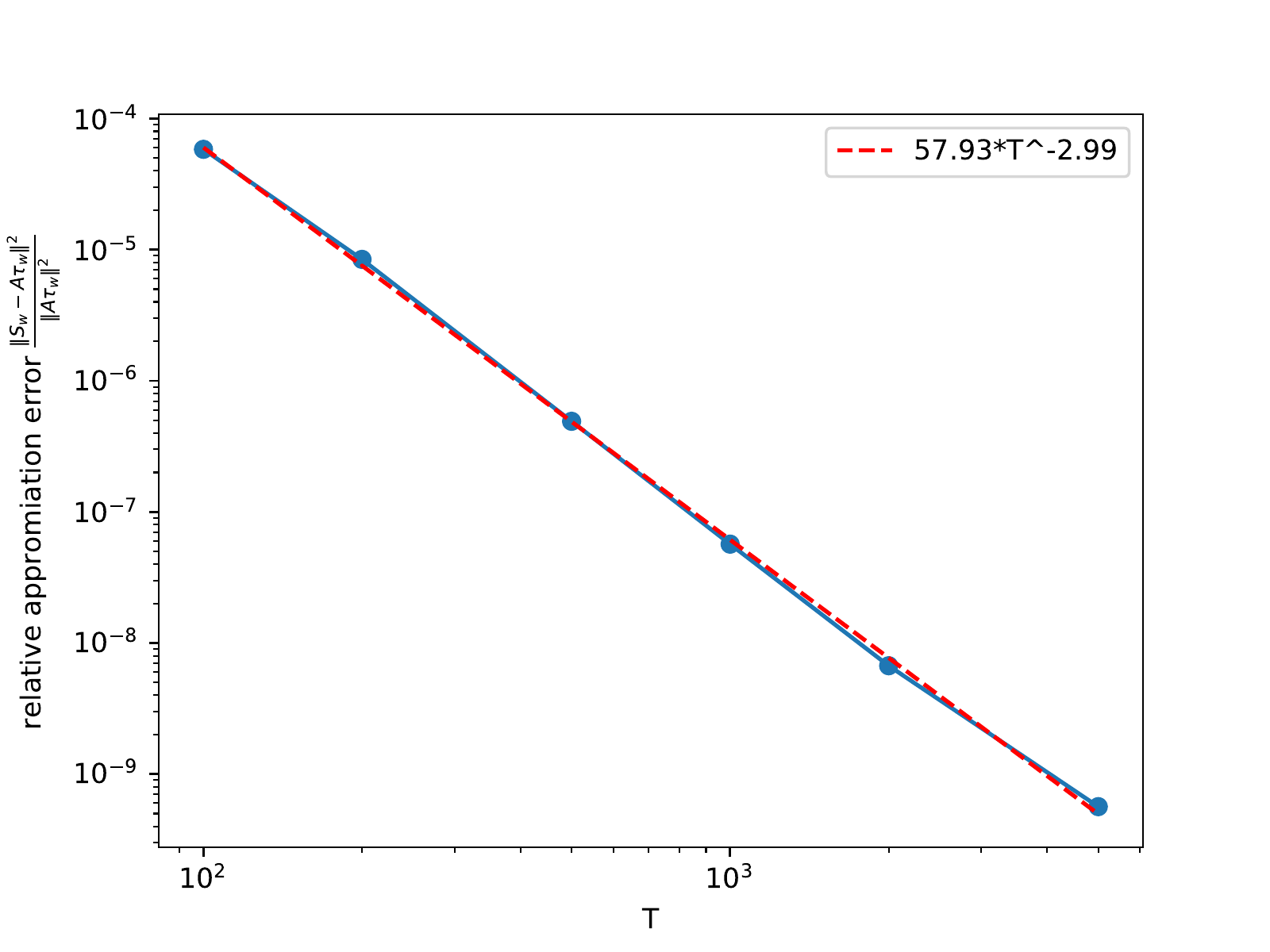}
    \end{minipage}
    \caption{
    \rev{
    Approximation error in Theorem \ref{thm:efficient_multiplicative_error} for a single ReLU neuron as a function of the ultraviolet truncation $T$. Values plotted are $T=100,200,500,1000,2000,5000$. Here, $n=8$ and $d=3$. Here, the number of quadrature points $N=10^4$ and the infra-red cutoff $t=0.1$ is kept constant.}
   }\label{fig:err1}
\end{figure}
\begin{figure}[h]
    \centering
    \begin{minipage}[c]{.9\textwidth}
    \includegraphics[width=.9\textwidth]{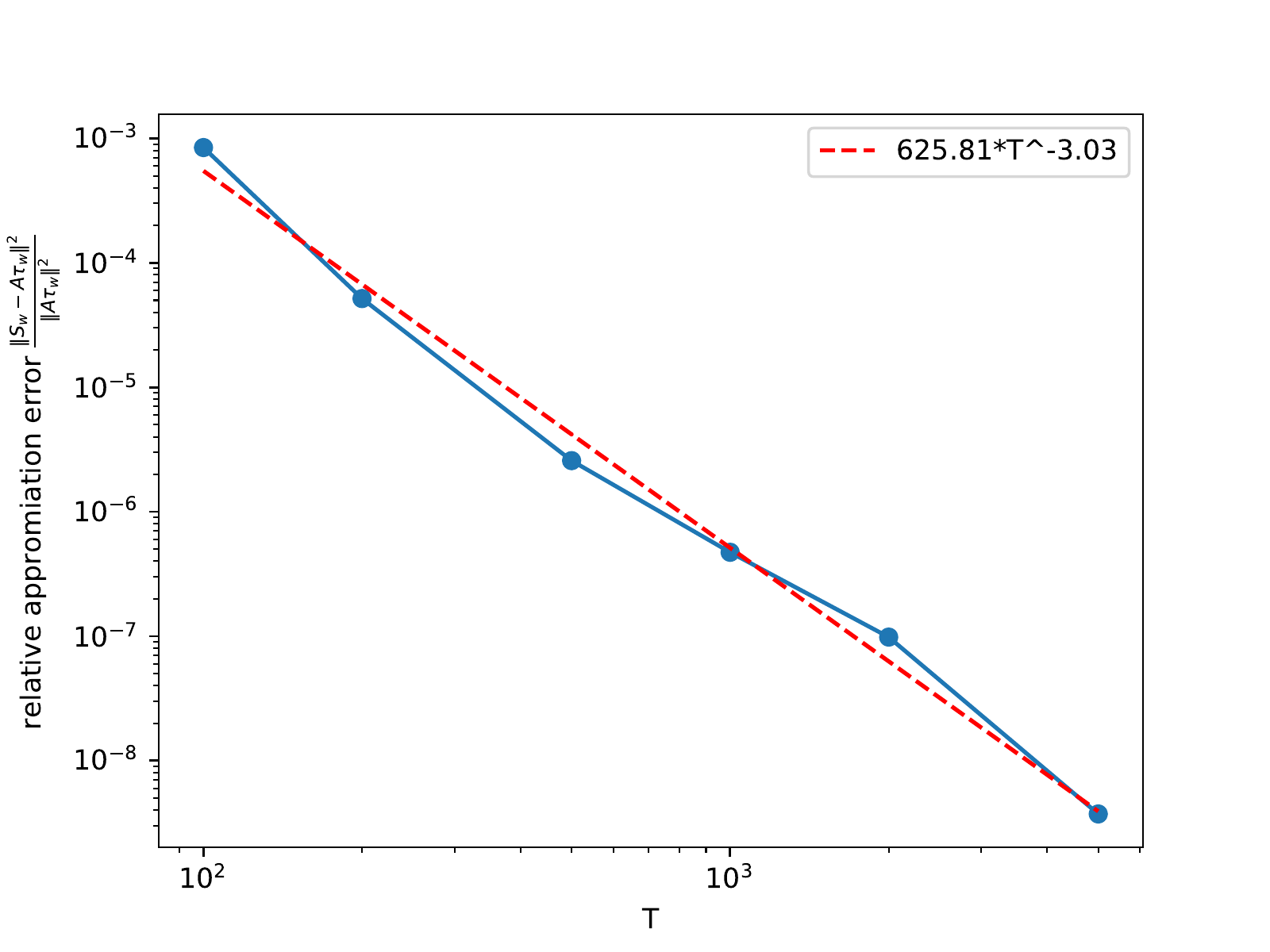}
    \end{minipage}
    \caption{
    \rev{
    The approximation error of Theorem \ref{thm:efficient_multiplicative_error} as in Figure \ref{fig:err1}, except the number of quadrature points $N$ is chosen to grow with the ultraviolet cutoff $T$ (specifically, we take $N=T$).
   }}
   \label{fig:err2}
\end{figure}

The cubic convergence with $T$ observed in Figures \ref{fig:err1} and \ref{fig:err2} is in accordance with our theory, because the anti-symmetrization operator $\mathcal A$ is approximately an isometry for highly oscillating functions, i.e., in the within the ultraviolet part. By Plancherel's equality we can approximate the squared error introduced by the tail truncation as the squared $L^2$ norm of the truncated tail, which is of order $\int_{|\theta|>T}|\hat\tau(\theta)|^2d\theta\propto\int_{|\theta|>T}|\theta|^{-4}\propto T^{-3}$ when $\tau$ is the ReLU activation.   
\end{revs}
%}

\section{Empirical generalization to multi-layer networks}

It is natural to ask whether the advantage of rough activation functions against cancellations remains as the depth of the neural network grows. We consider networks of depth $L=3,4,5$ and compare $\XnormN{\AS f}^2$ between two choices of activation functions: The smooth $\tanh$ and the rough \emph{normalized double ReLU} (DReLU) $\activation_\kappa(y)=\kappa\max\{-1,\min\{1,y\}\}$ where $\kappa\approx0.875$ is chosen such that $\EE[|\activation_\kappa(Z)|^2]=\EE[|\tanh(Z)|^2]$ for standard-Gaussian $Z\sim\cN(0,1)$. 
\begin{figure}[H]
    \centering
    \includegraphics[scale=.8]{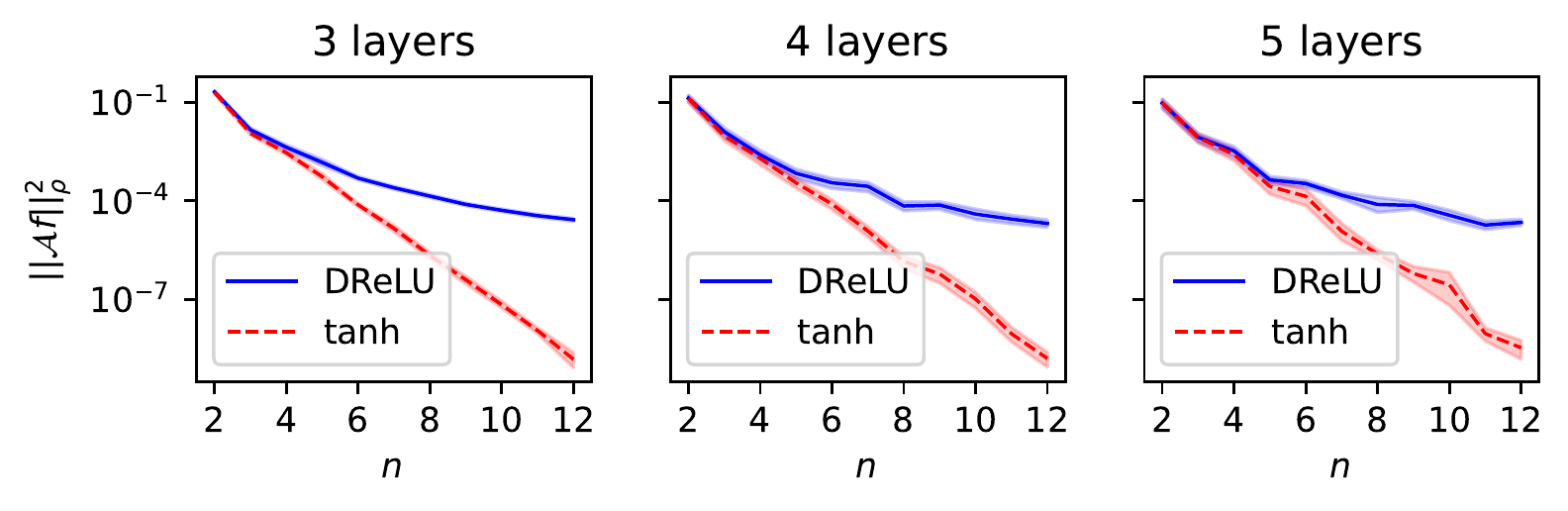}\hphantom{aaaaa}
    \vspace{-1em}\caption{Log-scale comparison between $\EE[\XnormN{\AS f}^2]$ for smooth and rough activation functions (tanh vs normalized DReLU) for antisymmetrized neural networks of different depths. Shaded areas show $90\%$ confidence regions.}\label{fig:depth}
\end{figure}

We take $d=3$, let all layers have width $m=3n$, and instantiate NNs from the Xavier initialization (independent Gaussian weights with variance $1/m$ where $m$ is the number of neurons in the preceding layer). \cref{fig:depth} shows that the rough activation function maintains its advantage for networks with more layers. 
\begin{figure}[H]
    \centering
    \includegraphics[scale=.65]{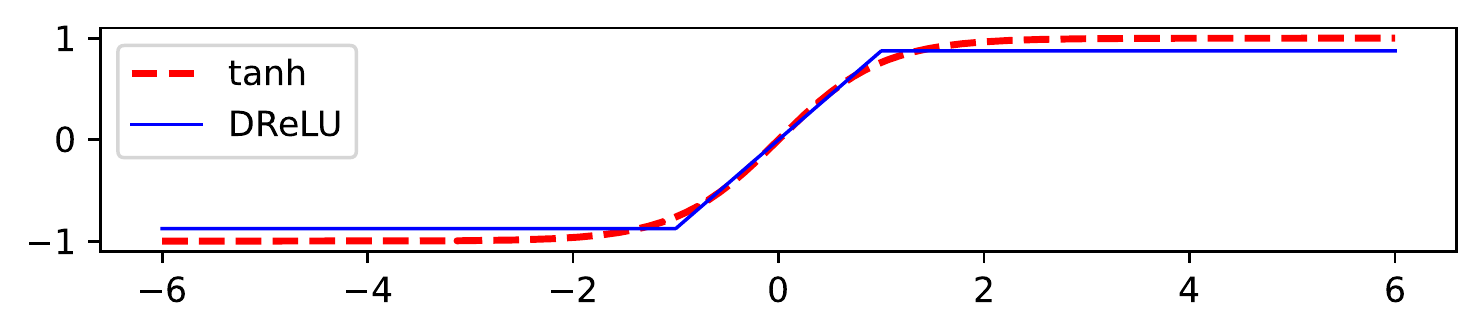}
    \caption{Tanh and the normalized DReLU. }
\end{figure}

\section{Conclusion}

Using the Fourier representation of the activation function, we observe that a rough activation function is necessary to tame the near-exact cancellations when antisymmetrizing two-layer NNs initialized with the standard initializations. Equivalently, an architecture based on a smooth activation function would require an initialization of the weights in the first layer distinct from the standard Xavier/He initializations to avoid the sign problem in the antisymmetric setting.
The Fourier perspective also provides a polynomial-time algorithm for approximately evaluating  explicitly antisymmetrized two-layer NNs. 
It may be possible that explicitly antisymmetrized two-layer NNs provides a path towards universal approximation of a class of antisymmetric functions without suffering from curse of dimensionality.
Our work also raises intriguing open questions about how the cancellations and efficient algorithms generalize to antisymmetrized multi-layer NNs \rev{as well as to the training regime.}

\section*{Acknowledgment}
This work was partially supported by the NSF Quantum Leap Challenge Institute (QLCI) program through grant number OMA-2016245 (N.A.), by the U.S. Department of Energy, Office of Science, Office of Advanced Scientific Computing Research and Office of Basic Energy Sciences, Scientific Discovery through Advanced Computing (SciDAC) program, and by the  Applied Mathematics Program of the US Department of Energy (DOE) Office of Advanced Scientific Computing Research under contract number DE-AC02-05CH1123 (L.L.). L.L. is a Simons Investigator. The authors thank Jeffmin Lin and Lexing Ying for helpful discussions.

%%%%%%%%%%%%%%%%%%%%%%%%%%%%%%%%%%%%%%%%%%%%%%%%%%%%%%%%%%%%%%%%%%%%%%%%%%%%%%%%%%%%%%%%%%%%%%%%%%%%

%\bibliographystyle{plain}
%\bibliography{refs}

\bibliographystyle{unsrt}
\bibliography{refs}

%\showfornips{
%\input{checklist}
%}

%\newpage

\appendix

\section{Lower-bounding the typical separation}
\label{sec:typ}

%\thmspace
%\lemmasingleneuron*
%\thmspace
%\begin{proof}
%Expand $
%\Xnorm{\AS f_{W,a}}^2=\sum_{k,l=1}^ma_ka_l\bracket{\AS\activation_{w^{(k)}}}{\AS\activation_{w^{(l)}}}$. 
%The $a_k\sim\cN(0,\cxh/m)$'s are independent so $\EE[a_ka_l]=\delta_{kl}\cxh/m$, and
%\begin{equation}
%\EE[\Xnorm{\AS f_{W,a}}^2\:|\:W]=\frac{\cxh}{m}\sum_{k,l=1}^m\delta_{kl}\EE[\bracket{\AS\activation_{w^{(k)}}}{\AS\activation_{w^{(l)}}}]=\frac{\cxh}m\sum_{k=1}^m\Xnorm{\AS\activation_{w^{(k)}}}^2.\end{equation}
%Taking the expectation over $W$ yields $\EE_{W,a}[\Xnorm{\binaryAS f_{W,a}}^2]=
%\cxh\EE[\Xnorm{\AS\activation_v}^2]
%$.
%\end{proof}

\thmspace
\begin{lemma}\label{lem:deltalowerbound}
For $w\sim\cN(0,\frac2{nd}I_{nd})$ sampled from the He initialization,
\begin{align}
\prob(\delta_w<\delta)
\le2\binom{n}{2}\Big(\frac{2nd\delta^2}{\cxh\pi}\Big)^{d/2}|B_d|,
\end{align}
where $|B_d|$ is the volume of the unit ball in $\RR^d$. For constant $d$ this is $O(n^{2+d/2}\delta^d)$.
\end{lemma}
\thmspace

\begin{proof}
If $w_i$ are sampled independently from a distribution $\Wdistr$ on $\RR^d$, then  
\begin{align}
\begin{split}
\prob(\delta_w<\delta)
&\le\sum_{i<j}\prob(\frac12\|w_i-w_j\|<\delta)+\prob(\frac12\|w_i+w_j\|<\delta)
\\&\le\sum_{i<j}\max_{w'\in\RR^d}\prob(\|w_i-w_j\|<2\delta\:|\:w_j=w')+\max_{w'\in\RR^d}\prob(\|w_i+w_j\|<2\delta\:|\:w_j=w')
\\&\le2\binom{n}{2}(2\delta)^d|B_d|\|\Wdistr\|_\infty,
\end{split}
\end{align}
where $\|\Wdistr\|_\infty$ is the supremum of the density.
For the He initialization $w_i\sim\cN(0,\frac2{nd}I_{d})$ we have $\|\Wdistr\|_\infty=(2\pi\cdot\frac{\cxh}{nd})^{-d/2}=(\frac{2\cxh\pi}{nd})^{-d/2}$.
\end{proof}

For fixed $d$, $v=w^{(k)}$ satisfies that $\prob(\delta_v<\delta)=O(n^{2+d/2}\delta^d)$. Then with probability $1-o(1)$, $v$ has \emph{typical separation} as defined in \cref{def:typsep}.

\begin{proof}[Proof of \cref{lem:highprob}]
Recall that by definition, $W$ is typical if  each $w^{(k)}$ is typical and at least half the $w^{(k)}$ have typical separation.
The distribution of
$\frac{nd}{\cxh}\|w^{(k)}\|^2\sim\chi^2(nd)$ implies that for each $k=1,\ldots,m$, $\prob(\|w^{(k)}\|^2<\cxh/2)\le (\sqrt e/2)^{nd/2}$ and $\prob(\|w^{(k)}\|^2>2\cxh)\le(2/e)^{nd/2}$ since $\EE[\|w^{(k)}\|^2]=\cxh$. By a union bound we have $\cxh/2\le\|w^{(k)}\|^2\le2\cxh$ for all $k=1,\ldots,m$ with probability $1-2m2^{-\Omega(nd)}$.
Furthermore,
$w^{(k)}_{ij}\sim\cN(0,\frac\cxh{nd})$ and a union bound imply that 
\[
\prob(|w^{(k)}_{ij}|\ge t\text{ for some }i,j,k)\le2mnde^{-\frac{ndt^2}{2\cxh}}=O\left(e^{-\frac{ndt^2}{2\cxh}+(C'+1)\log n+\log d}\right).
\] Given any $C''>0$ we may let $t=2\sqrt{C'+C''+1}\sqrt{\frac{\log(nd)}{nd}}$ and obtain that $\|w^{(k)}\|_\infty\le t$ for all $k$ with probability at least $1-n^{-C''}$.  This shows that each $w$ is typical with probability $1-1/n$ for appropriate $C$ in \cref{def:typical}.

If $\prob(\delta_{w^{(k)}}\ge\delta)\to1$ for each fixed $k$ then the probability that $\delta(w^{(k)}\ge\delta)$ for at least half the $k=1,\ldots,m$ also converges to $1$. Eq.~\eqref{whp} follows from \cref{lem:singleneuron_E}.
\end{proof}

\section{Bound on complex-exponential overlap kernel for the Gaussian envelope}
\label{sec:det_upper_bound}

In this section we recall the proof of the bound on a determinant of exponentials given in \cite{abrahamsen2023antisymmetric}.
\thmspace
\begin{restatable}{lemma}{lemmalowranksum}
    \label{lem:rank_one_terms}
Let $v=(v_1,\ldots,v_n)^T\in\RR^{n\times d}$ and $w=(w_1,\ldots,w_n)^T\in\RR^{n\times d}$. Then $(e^{v_i\cdot w_j})_{ij}=\sum_{k=0}^\infty Q_k$ where
\begin{equation}\opn{rank}Q_k\le\binom{k+d-1}{d-1},\qquad\|Q_k\|\le \frac{n(\|v\|_\infty\|w\|_\infty d)^k}{k!}.\label{Qk}\end{equation}
\end{restatable}
\thmspace
%\lemmalowranksum*
\begin{proof}
Let $(c_1,\ldots,c_d)$ and $(\tilde c_1,\dots,\tilde c_d)$ be the columns of $v$ and $w$ and let $\entwise$ denote elementwise operations. 
\begin{equation}
    \label{columndecomp}
(e^{v_i\cdot w_j})_{ij}=e^{\entwise\sum_{i=1}^dc_i\tilde c_i^T}=\entp_{i=1}^d e^{\entwise c_i\tilde c_i^T}.
\end{equation}
We first consider each factor $e^{\entwise c_i\tilde c_i^T}$ separately. Elementwise multiplication of rank-one matrices given as outer products corresponds to elementwise multiplication of the vectors, $ab^T\entp \tilde a\tilde b^T=(a\entp\tilde a)(b\entp\tilde b)^T$. Therefore, applying the Taylor expansion entrywise,
\begin{equation}
    \label{singlecolumn}
e^{\entwise c\tilde c^T}
=
\sum_{k=0}^\infty\frac{(c\tilde c^T)^{\entwise k}}{k!}
=
\sum_{k=0}^\infty\frac{(c^{\entwise k})(\tilde c^{\entwise k})^T}{k!},
\end{equation}
where $c=c_i$, $\tilde c=\tilde c_i$ are column vectors. Apply \eqref{singlecolumn} to each factor of \eqref{columndecomp},
\begin{align}
\entp_{i=1}^de^{\entwise c_i\tilde c_i^T}
&=
\sum_{k_1,\ldots,k_d=0}^\infty\frac{(\entp_{i=1}^d c_i^{\entwise k_i})(\entp_{i=1}^d \tilde c_i^{\entwise k_i})^T}{\prod_{i=1}^dk_i!}.
\\&=
\sum_{k=0}^\infty
\sum_{k_1+\ldots+k_d=k}^\infty\frac1{k!}\binom{k}{k_1,\ldots,k_d}(\entp_{i=1}^d c_i^{\entwise k_i})(\entp_{i=1}^d \tilde c_i^{\entwise k_i})^T\label{doublesum}
%\\&=\sum_{k=0}^\infty Q_k,
\end{align}
Let $Q_k$ be the innermost sum of \eqref{doublesum}. We estimate the maximum over the entries,  
\[\|Q_k\|\submax\le\frac{\|v\|_\infty^{k}\|w\|_\infty^k}{k!}\sum_{k_1+\ldots+k_d=k}^\infty\binom{k}{k_1,\ldots,k_d}=\frac{\|v\|_\infty^k\|w\|_\infty^kd^k}{k!}\]
and apply the inequality $\|Q_k\|\le n\|Q_k\|\submax$.  
\end{proof}

\thmspace
\begin{lemma}\label{lem:eigbound}
Let $\lambda_0\ge\lambda_1\ge\ldots$ be the absolute values of the eigenvalues of $(e^{v_i\cdot w_j})_{ij}$ and let $\mu=\|v\|_\infty\|w\|_\infty d$. Then $\lambda_0\le ne^\mu$, and for $\mu\le1/2$,
\begin{equation}
    \label{binomkd}
    \lambda_L\le\frac{2n}{p!}\mu^p,\qquad L=\binom{p+d-1}{d},
\end{equation}
where the case $p=0$ of \eqref{binomkd} holds with the interpretation $L=\binom{d-1}{d}=0$, $\lambda_0\le ne^{1/2}\le 2n$.
\end{lemma}
\thmspace
\begin{proof}
From the identity
\[
\binom{p+d-1}{d}=1+d+\binom{d+1}{d-1}+\cdots+\binom{p+d-2}{d-1},
\]
there are 
\[1+d+\binom{d+1}{d-1}+\cdots+\binom{p+d-2}{d-1}\ge\rank Q_0+\cdots+\rank Q_{p-1}\]
eigenvalues in front of $\lambda_L$ where $L=\binom{p+d-1}{d}$, and we have used \cref{lem:rank_one_terms}. By the min-max principle,
\[\lambda_L\le\left\|\sum_{k=p}^\infty Q_k\right\|\le n\sum_{k=p}^\infty\frac{\mu ^k}{k!}=\frac{n}{p!}\sum_{k=p}^\infty\mu^k=\frac{n}{p!}\frac{\mu^p}{1-\mu}\le \frac{2n}{p!}\mu^p.\]
\end{proof}

%\propdetbound*
\begin{proof}[Proof of proposition \ref{prop:detbound}]
By \cref{lem:eigbound} and the assumptions on $p$ we have $\lambda_{\lfloor n/2\rfloor}\le\frac{2n}{p!}\mu^p\le\frac{\mu^p}{2n}$ and $\lambda_0\le 2n$ where $\mu=d\|v\|_\infty\|w\|_\infty$, so it follows that 
$|\det((e^{v_i\cdot w_j})_{ij})|\le\lambda_0^{n/2}\lambda_{n/2}^{n/2}\le(\mu^p)^{n/2}=(\nu/2)^{pn}$.
\end{proof}

Apply the bound to $\DddN(\theta w,\theta w)$ with $p=\Theta(n^{1/d})$ to get that for $\theta\le t:=(2\sqrt{d} \|w\|_\infty)^{-1}$,
\begin{equation}\label{Dttlowbound}
\DttN{w}(\theta,\theta)=(\tfrac12\theta/t)^{\Omega(n^{1+1/d})}.\end{equation}

\label{sec:smooth}
Eq.~\eqref{FdefA} implies the \emph{triangle inequality}  $\Xnorm{\AS\activation_w}\le\frac1{\sqrt{2\pi}}\int_{-\infty}^\infty|\isoF\activation(\theta)|\sqrt{\DttN{w}(\theta,\theta)}d\theta$ 
because $\sqrt{\Dtt{w}(\theta,\theta)}=\Xnorm{\AS e^{i\theta w\cdot\xfunc}}$ by definition. Apply this triangle inequality to the low-pass part and bound the integrand using \eqref{Dttlowbound} to cancel the pole of $|\isoF\activation|$ at $0$. We then obtain lemma \ref{lem:lpbound}

%We recall the proof of proposition \ref{lem:lpbound} in \cref{sec:det_upper_bound}.

%%%%%%%%%%%%%%%%%%%%%%%%%%%%%%%%%%%%%%%%%%%%%%%%%%%%%%%%%%%%%%%%%%%%%%%%%%%%%%%%%%%%%%%%%%%%%%%%%%%%

\thmspace
\lemmalpbound*
\thmspace
\begin{proof}
The lower bound on $t$ follows directly from its definition and the definition of $w$ being typical.

Bound $\DttN{w}$ as in \eqref{Dttlowbound} and write $|\isoF\activation(\theta)|=O(|\theta|^{-r}+1)$. The triangle inequality yields $\XnormN{\LP\activation{t}}=O(2^{-pn/2}\int_{-t}^t(|\theta|/t)^{pn/2}(|\theta|^{-r}+1)d\theta)=O(2^{-pn/2}t(t^{-r}+1))$ where $p=\Omega(n^{1/d})$. Here we have cancelled the pole $|\theta|^{-r}$ by writing $(|\theta|/t)^{pn/2}|\theta|^{-r}\le(|\theta|/t)^r|\theta|^{-r}=t^{-r}$ so that the integrand is bounded by $t^{-r}+1$.
\end{proof}

\section{Generalized definition of roughness}
\label{sec:generalrough}
\begin{definition}[Generalized rough activation functions]\label{def:generalrough}
$\activation:\RR\to\CC$ is \emph{generalized rough} if its Fourier transform $\isoF\activation$ has tail decay $K<\infty$, i.e., if $\tailsum\activation(t)=t^{-O(1)}$, and if additionally,
\begin{enumerate}
    \item\label{it:boundpointbytail}
$|\isoF\activation(\theta)|^2/\tailsum{\activation}(\theta)\to0$ as $\theta\to\infty$.
\item\label{it:previouslyrealcondition}
There exists $\gamma>0$ such that for $1/2\le\sigma\le2$ and all $\theta\in\RR\del[-1,1]$ where $\isoF\activation(\theta)\neq0$,
\begin{equation}\Real\Big(\frac{\EE[\:\isoF\activation(\theta\pm|Z|)\:]}{\isoF\activation(\theta)}\Big)\ge\gamma,\end{equation}
where $Z\sim\cN(0,\sigma^2)$. Here, ``$\pm$'' is taken the sign of $\theta$, and $\Real$ is the real part of a complex number.
\end{enumerate}
\end{definition}
The assumptions of \cref{it:boundpointbytail} and \cref{it:previouslyrealcondition} prevent activation functions with excessive negative correlations between nearby frequencies. 
The constraint on the standard deviation  $1/2\le\sigma\le2$ is related to the Xavier / He initialization in \cref{def:He}.

\section{The Fourier inversion formula holds for ReLU and tanh}
We include the derivation (see \cite{abrahamsen2023antisymmetric}) to verify \cref{Fdef} which asserts that $\LP\activation\epsilon=p+C_\epsilon+O(\epsilon g)$ where $p$ is a low-degree polynomial and $g$ is bounded by a polynomial. Here we have defined $\LP\activation\epsilon=\activation-\HP\activation\epsilon$.
\begin{enumerate}
    \item $\activation=\opn{ReLU}$: We first evaluate the high-pass part
\[\HP\activation{\epsilon}(y)=|y|/2-\frac{\cos(\epsilon y)}{\pi \epsilon }-\frac{y\opn{Si}(\epsilon y)}\pi,\]
where $\opn{Si}(y)=\int_0^y\frac{\sin s}sds$. Since $\opn{ReLU}(y)=|y|/2+y/2$,
\begin{align}
\LP\activation{\epsilon }(y)
&=y/2+\frac{\cos(\epsilon y)}{\pi \epsilon }+\frac{y\opn{Si}(\epsilon y)}\pi.
\end{align}
Write $\LP\activation\epsilon=p+C_\epsilon+\varepsilon$ where
\[p(y)=y/2,\qquad C_\epsilon=\frac1{\pi\epsilon}.\]
Then the remainder satisfies
\begin{equation}\label{boundlp}
|\varepsilon|\le|\frac{\cos(\epsilon y)-1}{\pi \epsilon }|+|\frac{y\opn{Si}(\epsilon y)}\pi|\le\frac{(\epsilon y)^2}{2\pi \epsilon }+\frac{y\cdot(\epsilon y)}{\pi}=\epsilon g(y),\qquad g(y):=\frac{3 }{2\pi}y^2.\end{equation}
\item $\activation=\tanh$:
We can write the low-pass part as an absolutely convergent integral as
\[\LP\activation{\epsilon }(y)=\frac1{\sqrt{2\pi}}\int_{-\epsilon }^\epsilon \frac{-i\sqrt{\pi/2}}{\sinh(\pi\theta/2)}(e^{i\theta y}-1)d\theta.\]
Let $p,C_\epsilon\equiv0$ and bound
\begin{align}
    |\LP\activation{\epsilon }(y)|
   % \le&\frac1{\sqrt{2\pi}}\int_{-\epsilon }^\epsilon \Big|\frac{-i\sqrt{\pi/2}}{\sinh(\pi\theta/2)}(e^{i\theta y}-1)\Big| d\theta.
   % \\
    \le&\frac1{\sqrt{2\pi}}\int_{-\epsilon }^\epsilon \frac{\sqrt{\pi/2}}{|\pi\theta/2|}|\theta y| d\theta
    =
    \epsilon g(y),\qquad g(y):=\frac2\pi |y|.
\end{align}
\end{enumerate}

\end{document}